%% file: main.tex
\documentclass[10pt,twocolumn,letterpaper]{article}

\usepackage{cvpr}              
\usepackage{multirow}
\usepackage{colortbl}
\usepackage{pifont}
\usepackage{algorithm}
\usepackage{algpseudocode}
\usepackage{amsthm}
\input{preamble}
\definecolor{cvprblue}{rgb}{0.21,0.49,0.74}
\usepackage[pagebackref,breaklinks,colorlinks,allcolors=cvprblue]{hyperref}

\def\paperID{9402} 
\def\confName{CVPR}
\def\confYear{2026}

\title{Model Merging in the Essential Subspace}

\author{
Longhua Li$^{1,2}$ \quad Lei Qi$^{1,2}$\thanks{Corresponding authors.} \quad Qi Tian$^{3}$ \quad Xin Geng$^{1,2}$\footnotemark[1] \\
$^1$School of Computer Science and Engineering, Southeast University, Nanjing, China\\
$^2$Key Laboratory of New Generation Artificial Intelligence Technology and Its Interdisciplinary\\Applications (Southeast University), Ministry of Education, China\\
$^3$Huawei Technologies, Shanghai, China\\
{\tt\small lhli@seu.edu.cn, qilei@seu.edu.cn, tian.qi1@huawei.com, xgeng@seu.edu.cn}
}

\begin{document}
\maketitle
\input{sec/0_abstract}    
\input{sec/1_intro}

\input{sec/2_related_work}

\input{sec/3_method}

\input{sec/4_experiments}
\input{sec/5_conclusion}
{
    \small
    \bibliographystyle{./ieeenat_fullname}
    \bibliography{./main}
}

\input{sec/X_suppl}

\end{document}

%% file: sec/0_abstract.tex
\begin{abstract}
Model merging aims to integrate multiple task-specific fine-tuned models derived from a shared pre-trained checkpoint into a single multi-task model without additional training. Despite extensive research, task interference remains a major obstacle that often undermines the performance of merged models.
In this paper, we propose ESM (\textbf{E}ssential \textbf{S}ubspace \textbf{M}erging) , a robust framework for effective model merging.
We begin by performing Principal Component Analysis (PCA) on feature shifts induced by parameter updates. The resulting principal directions span an essential subspace that dominantly influences feature representations. Each task's parameter update matrix is projected onto its respective essential subspace for low-rank decomposition before merging. This methodology mitigates inter-task interference while preserving core task-specific functionality.
Furthermore, we introduce a multi-level polarized scaling strategy that amplifies parameters containing critical knowledge and suppresses redundant ones, preventing essential knowledge from being overwhelmed during fusion.
Extensive experiments across multiple task sets and model scales demonstrate that our method achieves state-of-the-art performance in multi-task model merging.
Code is available at \url{https://github.com/kiddo127/ESM}.
\end{abstract}



%% file: sec/1_intro.tex
\section{Introduction}
\label{sec:intro}

In recent years, the pre-training–fine-tuning paradigm has revolutionized performance across numerous downstream tasks, enabling models to excel by adapting to task-specific data. This success has spurred the creation of many specialized expert models. The growth of public model repositories is now fueling demand for techniques that reuse and integrate these models. Among them, model merging \cite{wortsman2022model,ilharcoediting,yan2025calm,sun2025towards} has emerged as a promising, parameter-efficient approach to combine multiple experts into a single versatile model, thus marking a pivotal advance toward the development of general-purpose artificial intelligence.

However, integrating the capabilities of multiple models remains challenging. A key difficulty is inter-task interference, as fine-tuning on different tasks often drives parameters in conflicting directions. Simple averaging methods such as Model Soup \cite{wortsman2022model} mix these updates, diluting task-specific knowledge and leading to suboptimal performance.
To alleviate this problem, subsequent studies analyzed task vectors, defined as the parameter differences between fine-tuned and pre-trained models \cite{ilharcoediting,yadav2023ties,matena2022merging,jindataless}. More recently, advanced methods have employed Singular Value Decomposition (SVD) on task vectors to identify low-rank subspaces that compactly represent task knowledge and reduce redundancy during merging \cite{stoicamodel,gargiulo2025task,marczakno,zhang2026dc}. However, these SVD-based subspaces are not well aligned with task feature spaces and exhibit limited knowledge concentration.


In this paper, we first theoretically analyze the limitations of SVD-based low-rank decomposition for task-specific parameter updates. Although these methods truncate the smallest singular values, they overlook the feature distribution. Consequently, when input tokens are aligned with the truncated singular vectors, significant truncation errors may still occur, as illustrated in Equation \ref{eq:svd_loss}.
Moreover, when merging a large number of task-specific update parameters, the core knowledge of each task may be overwhelmed. As shown in Figure \ref{fig:single_task_sequential_layer_load} and \ref{fig:sequential_layer_merge}, we observe that the magnitude (norm) of task parameter updates is closely correlated with their directional importance: large-magnitude updates often correspond to critical adjustments essential for capturing task-specific or shared knowledge across tasks, whereas small-magnitude updates, though potentially negligible for individual task performance, may obscure important task knowledge after model merging.

Building on these insights, our method consists of two components. Instead of directly decomposing the task matrix (i.e., the parameter update matrix), we first perform Principal Component Analysis (PCA) on the proxy activation shifts induced by the task matrix. The matrix is then decomposed along the resulting principal directions, forming the task’s Essential Subspace, which is directly connected to the output feature space and ensures minimal feature loss compared to other low-rank subspaces of the same rank. Each task matrix is subsequently projected onto its Essential Subspace before fusion. Experiments demonstrate that this subspace aligns more closely with the task’s feature representation, exhibits stronger energy concentration, and isolates task-specific knowledge more effectively than conventional SVD-based subspaces.
Next, we introduce a simple yet effective Polarized Scaling mechanism. It adaptively rescales task matrices at three levels: across layers, tasks, and dimensions, amplifying high-norm (strong-signal) components while suppressing low-norm (weak-noise) ones. This polarization effectively reduces inter-task interference during model fusion and ensures the preservation of essential task knowledge in the merged model.

Our main contributions are summarized as follows:
\begin{itemize}
    \item We propose a decomposition method that builds an essential subspace aware of activation shift distributions, aligning closely with feature distribution and concentrating energy more effectively, thereby better preserving task-specific knowledge and reducing cross-task interference.
    \item We show that the pre-fusion and post-fusion task matrix norms reflect gradient confidence and inter-task consensus, respectively. Based on this, we introduce a multi-level polarized scaling mechanism that prevents crucial parameters from being overshadowed during merging.
    \item Extensive experiments show that our method outperforms existing model merging methods and effectively narrows the gap between the merged model and the expert models.
\end{itemize}

%% file: sec/2_related_work.tex
\section{Related Work}
\label{sec:related_work}

\subsection{Model Merging}
Model merging aims to combine multiple task-specific fine-tuned models into a unified multi-task model without retraining.
To ensure merging stability, recent studies typically merge models that are fine-tuned from the same pre-trained checkpoint.
Model Soup \cite{wortsman2022model} averages fine-tuned weights to improve generalization, while Task Arithmetic \cite{ilharcoediting} introduces task vectors, defined as the parameter differences between fine-tuned and pre-trained models, to enable vector-based knowledge composition. However, direct averaging of task vectors often causes severe task interference due to conflicting updates. To address this, TIES-Merging \cite{yadav2023ties} trims redundant parameters before averaging salient ones, AdaMerging \cite{yangadamerging} learns adaptive task-wise coefficients, and DARE \cite{yu2024language} resets redundant updates while rescaling the rest. Information-weighted methods such as Fisher Merging \cite{matena2022merging} and RegMean \cite{jindataless} use Fisher information or input similarity for weighted averaging. Other works refine merging through parameter- or layer-wise strategies \cite{du2024parameter,zhang2024knowledge,wanglines}, or leverage implicit or modular representations to enhance flexibility \cite{chengwhoever,huang2024emr,zheng2025free}.

Recent advances move beyond raw parameter space to the spectral domain. TSV-M \cite{gargiulo2025task} perform Singular Value Decomposition (SVD) on task matrices and merge along the top singular directions that capture dominant functional subspaces. Iso-CTS \cite{marczakno} constructs an isotropic common subspace through singular value normalization followed by task-specific refinements, achieving state-of-the-art performance. However, singular values reflect only the parameter energy rather than their functional impact.
To overcome this limitation, we decompose each task matrix within an essential subspace derived from its effect on output activations, which better captures task-specific features. This leads to lower truncation error and improved preservation of task knowledge during merging. Moreover, we introduce a polarized scaling mechanism that amplifies strong signals and suppresses noise, further enhancing fusion performance.

\subsection{Model Weight Transformation}
Transforming model weights has been extensively studied in areas such as model compression \cite{wangsvd,li2025stratified,huang2026few} and efficient knowledge transfer \cite{xia2024transformer,xia2024exploring,li2025one,xu2025learngene}. One of the most popular approaches to model weight transformation is based on the low-rank assumption of weight matrices.
The LoRA family of methods \cite{hulora,dettmers2023qlora,ding2023parameter,WOS001539022300003} assumes that fine-tuning updates are inherently low-rank and learns compact matrices to parameterize these updates.
Other methods use SVD-based decompositions for parameter-efficient fine-tuning \cite{sun2022singular,han2023svdiff} or model compression \cite{li2023losparse,saha2023matrix,wangsvd,wang2025svd}. More recently, low-rank decompositions have been applied to model merging \cite{stoicamodel,gargiulo2025task,marczakno}, combining task updates in reduced subspaces to mitigate inter-task interference.

In contrast, our proposed Essential Subspace Decomposition (ESD) constructs the decomposition space not from the weight updates themselves but from the activation shifts induced by these updates. By capturing task-specific principal directions in the activation space, ESD produces sparse yet expressive task representations, reducing cross-task interference while preserving high task fidelity.

%% file: sec/3_method.tex
\begin{figure*}[t]
    \centering
    \includegraphics[width=\linewidth]{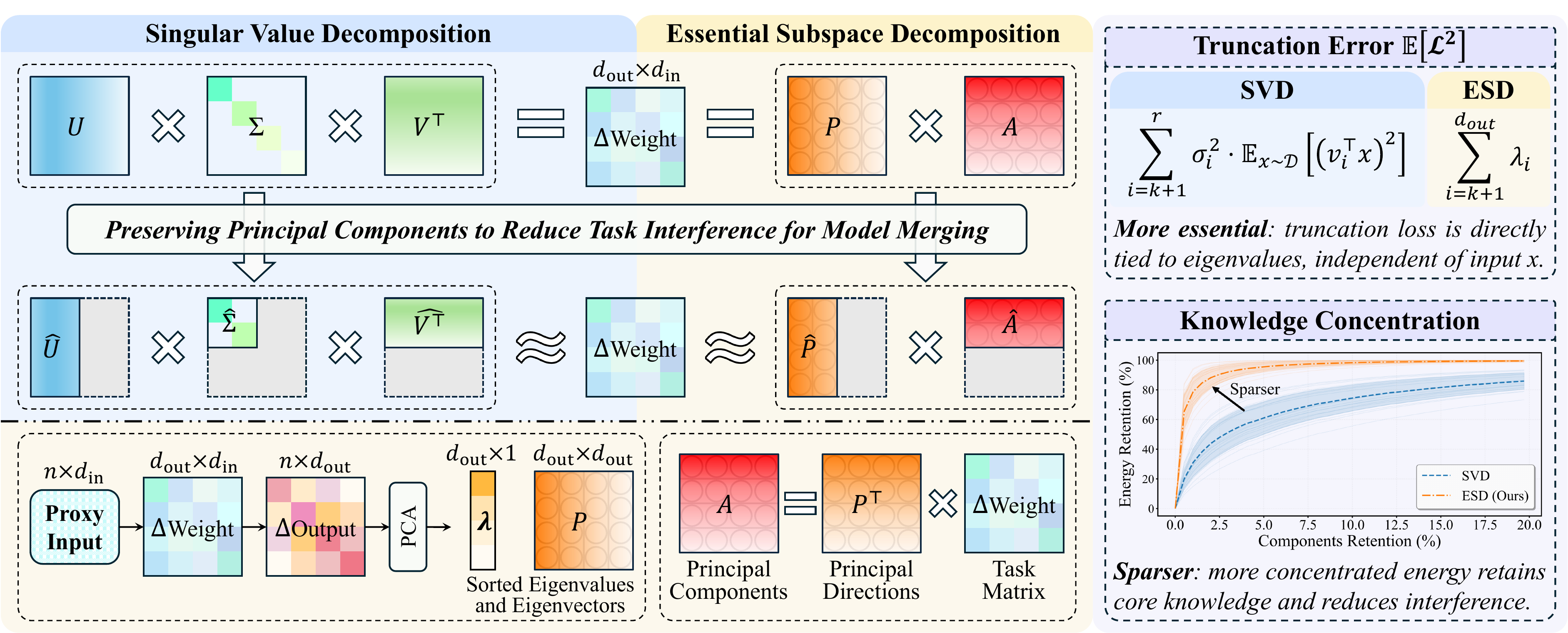}
    \caption{Essential Subspace Decomposition (ESD) versus Singular Value Decomposition (SVD). Unlike SVD, which decomposes the task matrix solely based on weights, ESD decomposes them based on feature shift distributions. When truncating components for merging, ESD's expected truncation error is directly related to the magnitude of the discarded eigenvalues and yields higher knowledge retention.}
    \label{fig:ESD_vs_SVD}
\end{figure*}

\section{Methodology}

\subsection{Preliminaries on Model Merging}
Model merging aims to integrate a collection of task-specific models, each fine-tuned from a common pre-trained checkpoint, into a single unified model without additional retraining. Formally, let $\theta_0$ denote the parameters of the pre-trained model, and $\theta_t$ be the parameters of the expert model fine-tuned on task $t$, where $t=1,\ldots,T$. The fundamental object of interest in model merging is the task update, which captures how fine-tuning shifts the model away from the pre-trained weights.
Following Task Arithmetic \cite{ilharcoediting}, the task vector for task $t$ is defined as:
\begin{equation}
\tau_{t}=\operatorname{Flatten}(\theta_{t}-\theta_{0}).
\end{equation}
Given the structured nature of models, it is preferable to retain the matrix form of the update rather than flattening it into a vector. The task matrix of layer $\ell$ is defined as:
\begin{equation}
\Delta_{W_t}^{ \left( \ell \right)}= \theta_{t}^{ \left( \ell \right)}- \theta_{0}^{ \left( \ell \right)}.
\end{equation}
Each $\Delta_{W_t}^{(\ell)}$ represents the task-specific parameter update at layer $\ell$, preserving the row–column structure essential for spectral analysis and subspace alignment. The goal of model merging is to construct merged weights $\theta_M$ that support all tasks, typically in the form:
\begin{equation}
\theta_{M}= \theta_{0}+f \left( \Delta_{W_1}, \cdots, \Delta_{W_T} \right),
\end{equation}
where $f \left( \cdot \right)$ is a merging function, which is the main focus of current model merging research \cite{ilharcoediting, stoicamodel,gargiulo2025task,marczakno}.

\subsection{Model Merging in Essential Subspace}
Unlike previous methods that merge models in the truncated singular vector subspace obtained via SVD, we propose to decompose and merge task matrices within a more essential subspace that is better aligned with the task's output feature space, as illustrated in Figure \ref{fig:ESD_vs_SVD}.
For simplicity, unless otherwise specified, we omit the layer index 
$\ell$ and task identifier $t$ in the task matrix and denote it simply as $\Delta_W$.

\paragraph{The Limitation of Feature Distribution-Agnostic SVD.}
Recent model merging methods often apply SVD to the task matrix $\Delta_{W} \in \mathbb{R}^{d_{\text{out}} \times d_{\text{in}}}$, keeping only the top-$k$ singular components to retain essential knowledge and reduce task interference, yielding the truncated approximation $\widehat{\Delta_{W}}$.
However, this parameter-centric view ignores the input feature distribution. While it minimizes the Frobenius norm reconstruction error of $\Delta_{W}$, it provides no guarantee about preserving the functional output. For an input $x\sim\mathcal{D}$, the expected output error after discarding the smallest $r-k$ singular components is (see Appendix \ref{sec:Proof_SVD_Truncation_Error} for derivation):
\begin{equation}\label{eq:svd_loss}
\mathbb{E}_{x\sim\mathcal{D}} \left[ \|\Delta_{W}x - \widehat{\Delta_{W}}x \|_2^2 \right] = \sum_{i=k+1}^r\sigma_{i}^{2} \cdot \mathbb{E}_{x\sim\mathcal{D}} \left[ (v_{i}^{\top}x)^2 \right],
\end{equation}
where $r$ denotes the number of non-zero singular values. As shown, the error depends not only on the discarded singular values $\sigma_i$, but also on the alignment between the input distribution and the right singular vectors $v_i$. A direction with small $\sigma_i$ may be functionally critical if inputs project strongly onto $v_i$. By ignoring the input distribution, SVD risks discarding functionally essential information.

\paragraph{Essential Subspace Decomposition (ESD).}
To address this limitation, we introduce Essential Subspace Decomposition (ESD), an input distribution-aware method that directly connects task matrix decomposition to its functional influence. Instead of relying on the parameter energy of the task matrix, ESD constructs a basis from the principal directions of activation shifts induced by the task matrix.
For each task $t$, we sample a lightweight proxy dataset (32 unlabeled examples in practice). By performing a forward pass through the task-specific fine-tuned model and recording the intermediate activations, we obtain the activation shift. Specifically, given $n$ input tokens of dimension $d_{\text{in}}$, forming $X_{\text{proxy}} \in \mathbb{R}^{n \times d_{\text{in}}}$, the shift is computed as:
\begin{equation}\label{eq:activation_shift}
\Delta_{O} = X_{\text{proxy}} \Delta_{W}^{\top} \in \mathbb{R}^{n \times d_{\text{out}}},
\end{equation}
which captures the functional footprint of $\Delta_{W}$ on a representative set of inputs. By performing PCA on $\Delta_{O}$, we obtain a set of eigenvectors $\boldsymbol{p}_i$ and their corresponding eigenvalues $\lambda_i$, sorted by explained variance. These eigenvectors form an orthonormal basis $P = [\boldsymbol{p}_1, \boldsymbol{p}_2, \dots, \boldsymbol{p}_{d_{\text{out}}}] \in \mathbb{R}^{d_{\text{out}} \times d_{\text{out}}}$ for the output space that is inherently aligned with the task matrix's functional behavior.
The original task matrix $\Delta_W$ is projected onto the space defined by $P$, which yields the coordinate matrix $A = P^{\top} \Delta_{W}\in\mathbb{R}^{d_{\text{out}}\times d_{\text{in}}}$, where $A$ contains the coordinates of $\Delta_W$ with respect to the basis $P$. Thus, $\Delta_W$ can be factorized as:
\begin{equation}
\Delta_{W} = PA = P(P^{\top} \Delta_{W}).
\end{equation}
We truncate to the top-$k$ principal components to form the essential basis $\hat{P}=[p_1, \dots, p_k] \in \mathbb{R}^{d_{\text{out}} \times k}$. The corresponding coordinate matrix is $\hat{A} = \hat{P}^{\top} \Delta_{W}\in\mathbb{R}^{k\times d_{\text{in}}}$, leading to the low-rank approximation:
\begin{equation}
\widehat{\Delta_{W}} = \hat{P} \hat{A} = \hat{P} (\hat{P}^{\top} \Delta_{W}).
\end{equation}
We can theoretically show (see Appendix \ref{sec:Proof_ESD_Truncation_Error} for derivation) that the expected output error under this decomposition is now decoupled from the input direction:
\begin{equation}\label{eq:esd_loss}
\mathbb{E}_{x\sim\mathcal{D}} \left[ \|\Delta_{W}x - \widehat{\Delta_{W}}x \|_2^2 \right] = \sum_{i=k+1}^{d_{\text{out}}} \lambda_i.
\end{equation}
Unlike SVD (Equation \ref{eq:svd_loss}), the ESD truncation error (Equation ~\ref{eq:esd_loss}) depends only on the sum of discarded eigenvalues. Removing small-eigenvalue components thus discards the least functionally relevant directions, making ESD truly ``essential'' by optimally preserving task expressiveness and enabling sparser, higher-fidelity model merging.

\paragraph{Essential Subspace Merging (ESM).}
Building upon the proposed ESD, we propose Essential Subspace Merging (ESM), a strategy that robustly aggregates task matrices by only considering the functionally important components. This method is analogous to TSV-M \cite{gargiulo2025task}, but is adapted to leverage our distribution-aware low-rank factors.
The ESM process follows three sequential steps:
\begin{enumerate}
    \item \textbf{Factorization and Truncation.} For each task $t \in \{1, \dots, T\}$, we first factorize the task matrix $\Delta_{W_t}$ into its essential basis $P_t$ and coordinate matrix $A_t$, such that $\Delta_{W_t} = P_t A_t$, where $P_t \in \mathbb{R}^{d_{\text{out}} \times d_{\text{out}}}$ and $A_t \in \mathbb{R}^{d_{\text{out}} \times d_{\text{in}}}$. Given $T$ tasks, we allocate a rank budget $k = \lfloor d_{\text{out}} / T \rfloor$ to each one. We truncate the task-specific factors to their top-$k$ components, resulting in the sparse factors $\hat{P}_t \in \mathbb{R}^{d_{\text{out}} \times k}$ and $\hat{A}_t \in \mathbb{R}^{k \times d_{\text{in}}}$.

    \item \textbf{Concatenation.} Next, we form the merged basis and coordinate matrices by horizontally and vertically concatenating the truncated factors across all tasks, respectively:
    \begin{equation}
    P_{\text{cat}} = [\hat{P}_1 | \hat{P}_2 | \dots | \hat{P}_T] \in \mathbb{R}^{d_{\text{out}} \times (k \cdot T)},
    \end{equation}
    \begin{equation}\label{eq:A_cat}
    A_{\text{cat}} = 
    \begin{bmatrix} 
    \hat{A}_1 \\ 
    \hat{A}_2 \\ 
    \vdots \\ 
    \hat{A}_T 
    \end{bmatrix} 
    \in \mathbb{R}^{(k \cdot T) \times d_{\text{in}}}.
    \end{equation}
    The initial merged matrix is simply the product $P_{\text{cat}} A_{\text{cat}}$.

    \item \textbf{Orthogonalization to Minimize Interference.} The concatenated factors $P_{\text{cat}}$ and $A_{\text{cat}}$ consist of components derived from different task subspaces, which may not be mutually orthogonal, leading to significant interference in the merged model. To reconstruct the merged matrix with minimized correlation, we must orthogonalize these concatenated factors. Following TSV-M \cite{gargiulo2025task}, we first compute the SVDs for each concatenated matrix:
    \begin{equation}
    P_{\text{cat}} = U_P \Sigma_P V_P^{\top}, \ 
    A_{\text{cat}} = U_A \Sigma_A V_A^{\top}.
    \end{equation}
    To ensure that more important parameter directions are preferentially preserved during orthogonalization, we apply eigenvalue-based weighting to both the directional vectors of $P_{\text{cat}}$ and the coordinate vectors of $A_{\text{cat}}$ prior to performing the above SVD.
    We then perform a whitening operation on the matrices by retaining only the orthogonal components $U$ and $V$. This transformation effectively rotating the basis to be maximally decorrelated:
    \begin{equation}
    \widetilde{P}_{\text{cat}} = U_P V_P^{\top}, \ 
    \widetilde{A}_{\text{cat}} = U_A V_A^{\top}.
    \end{equation}
    The final merged task matrix $\Delta_{W_{\text{merged}}}$ is then constructed by multiplying these orthogonalized factors:
    \begin{equation}
    \Delta_{W_{\text{merged}}} = \widetilde{P}_{\text{cat}} \widetilde{A}_{\text{cat}}.
    \end{equation}
\end{enumerate}

\subsection{Polarized Scaling for Knowledge Enhancement}
\label{sec:polarized_scaling}
Although ESD preserves task matrix functionality, intense weight competition during fusion can still overwhelm critical knowledge. We note that the norm of a weight block update reflects its directional confidence and inter-task consensus, serving as an effective scaling factor. We therefore introduce Polarized Scaling (Figure \ref{fig:Polarized_Scaling}) to amplify signals and suppress noise for improved merging performance.

\begin{figure}[t]
    \centering

    \begin{subfigure}{\linewidth}
        \centering
        \includegraphics[width=\linewidth]{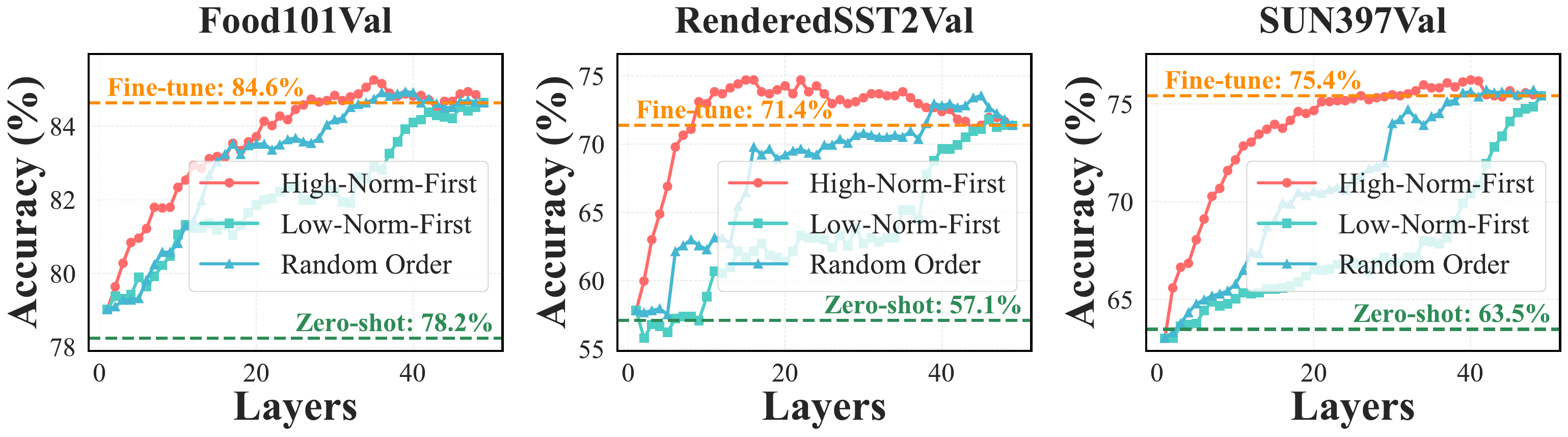}
        \caption{Sequential Task Matrices Loading}
        \label{fig:sequential_layer_load}
    \end{subfigure}
    

    \begin{subfigure}{\linewidth}
        \centering
        \includegraphics[width=\linewidth]{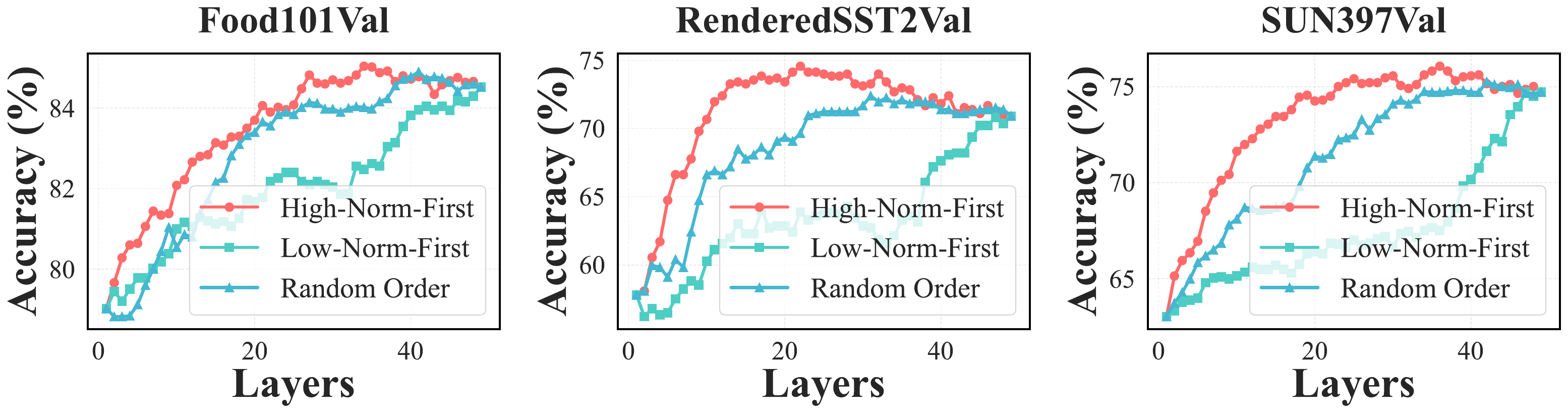}
        \caption{Sequential \textit{Averaged-Norm} Task Matrices Loading}
        \label{fig:sequential_mean_layer_load}
    \end{subfigure}
    \caption{Performance evaluation of layer-wise task matrix loading under different ordering strategies on a pre-trained ViT-B/32 backbone. (a) Direct loading of task matrices. (b) Loading with layer-wise norms averaged to reflect directional importance.}
    \label{fig:single_task_sequential_layer_load}
\end{figure}

\begin{figure}[t]
    \centering
    \begin{subfigure}{0.32\linewidth}
        \centering
        \includegraphics[width=\linewidth]{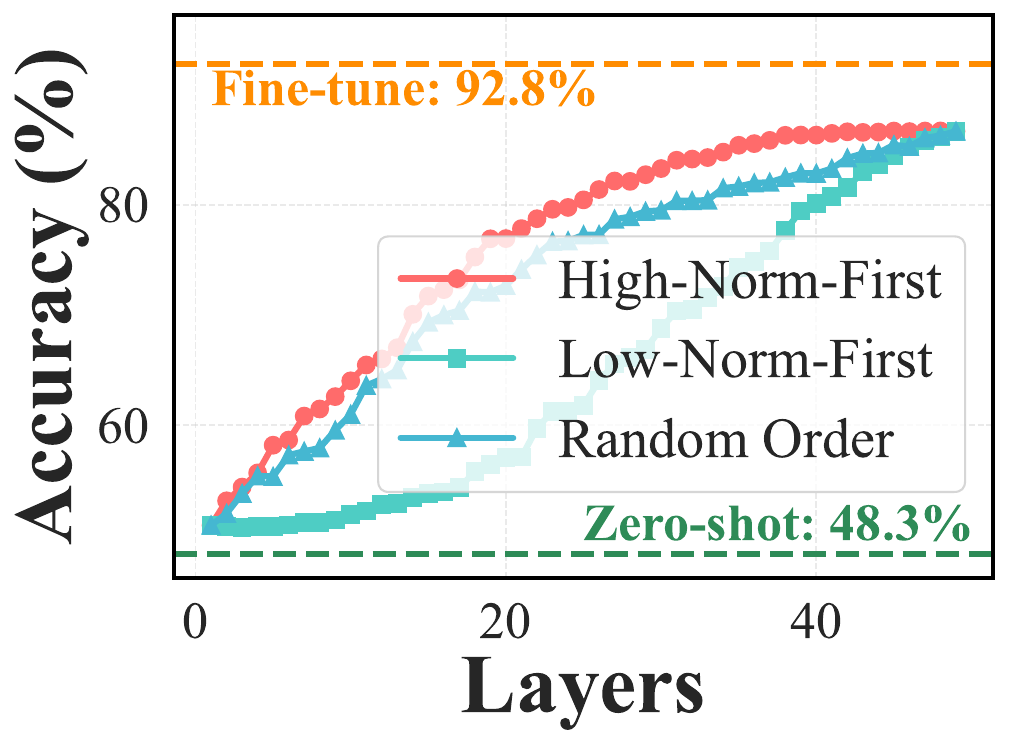}
        \caption{ViT-B/32}
    \end{subfigure}
    \hfill
    \begin{subfigure}{0.32\linewidth}
        \centering
        \includegraphics[width=\linewidth]{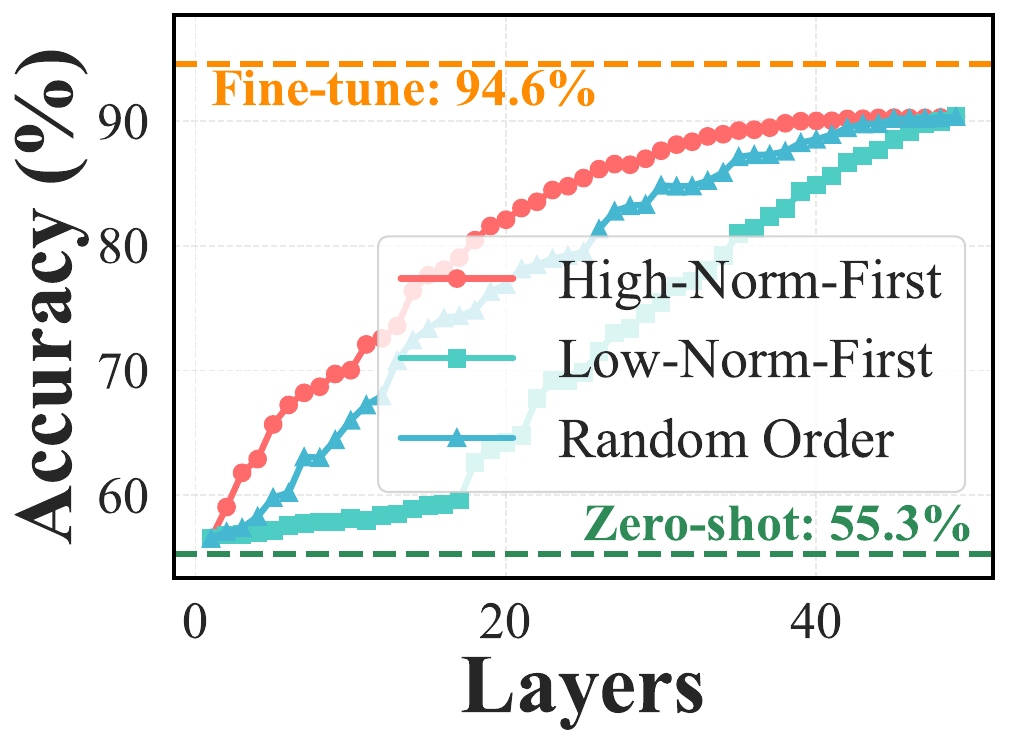}
        \caption{ViT-B/16}
    \end{subfigure}
    \hfill
    \begin{subfigure}{0.32\linewidth}
        \centering
        \includegraphics[width=\linewidth]{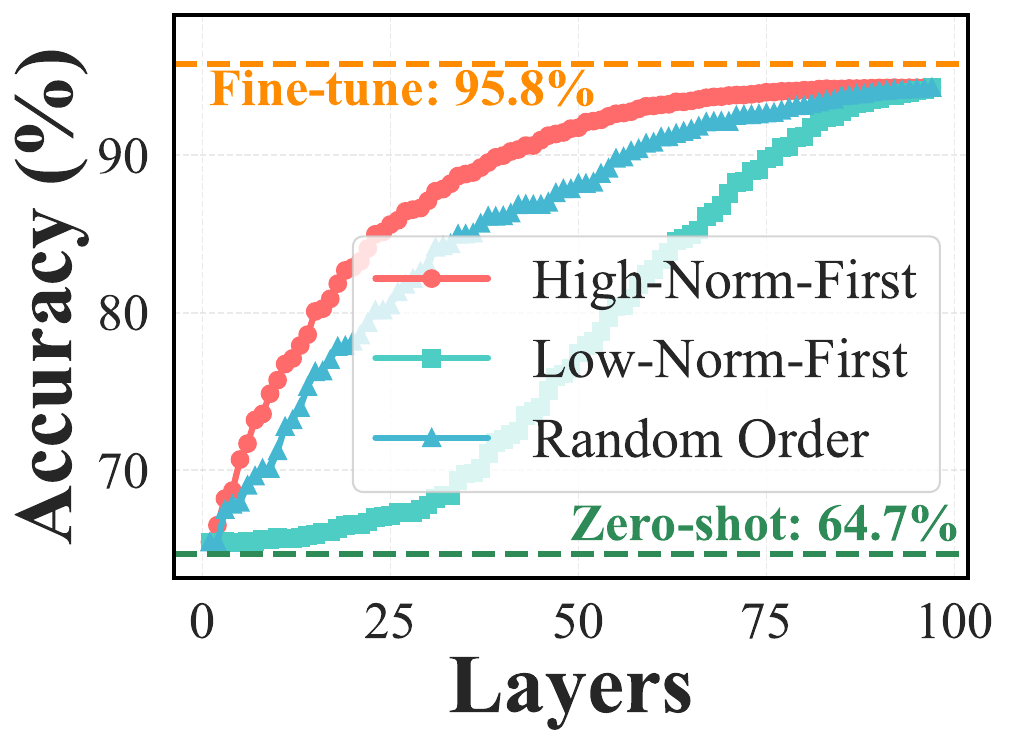}
        \caption{ViT-L/14}
    \end{subfigure}
    \caption{Performance evaluation on the 8-task benchmark when loading merged task matrices layer-by-layer into a pre-trained backbone under different ordering strategies.}
    \label{fig:sequential_layer_merge}
\end{figure}

\paragraph{Empirical Evidence: Single-Task Merging.}
As shown in Figure \ref{fig:sequential_layer_load}, we sequentially add task matrices from a fine-tuned model to the zero-shot base model, comparing three orders: high-norm-first, low-norm-first, and random. The results consistently indicate that adding high-norm matrices first yields the best performance. Moreover, incorporating low-norm matrices in later stages leads to clear performance degradation.
To isolate the effect of direction rather than scale, we repeat the experiment after normalizing each layer’s task matrix to the same average norm (Figure \ref{fig:sequential_mean_layer_load}). The high-norm-first order still performs best, revealing a key insight: \textit{high-norm task matrices matter because they capture consistent, task-aligned directions in the optimization landscape, encoding the task’s essential knowledge}.

\paragraph{Empirical Evidence: Multi-Task Merging.}
As shown in Figure \ref{fig:sequential_layer_merge}, we load the aggregated task matrices, obtained after Essential Subspace Merging (ESM), into the pre-trained model, again following the three loading orders based on the merged matrix norm. The results confirm that prioritizing high-norm components yields superior merged performance. This suggests that \textit{the high-norm components in the merged matrix are likely the inter-task consensus, representing shared, important knowledge directions}.

\begin{figure}[t]
    \centering

    \begin{subfigure}{\linewidth}
        \centering
        \includegraphics[width=\linewidth]{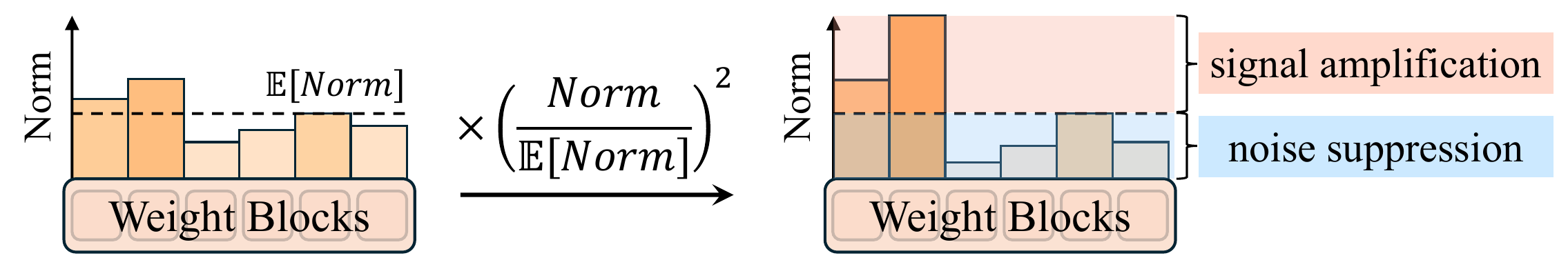}
        \caption{Polarized Scaling}
        \label{fig:Scaling_Method}
    \end{subfigure}
    

    \begin{subfigure}{\linewidth}
        \centering
        \includegraphics[width=\linewidth]{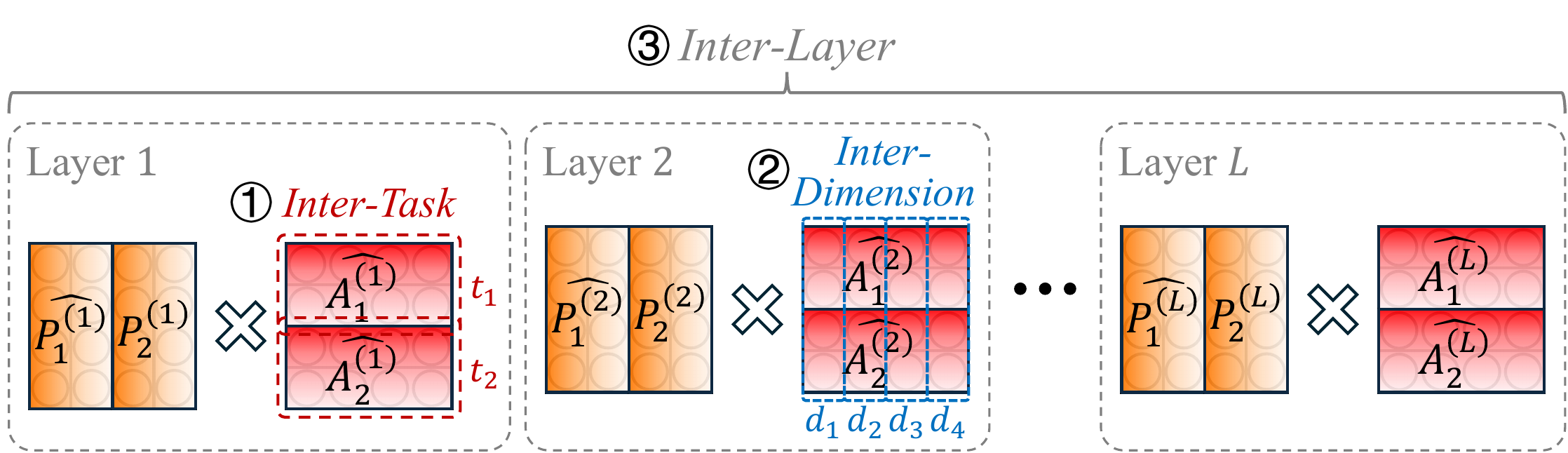}
        \caption{Scaling Hierarchy}
        \label{fig:Scaling_Perspectives}
    \end{subfigure}

    \caption{(a) The proposed Polarized Scaling uses the norm of parameter updates as scaling factors to amplify essential parameters submerged by redundant ones. (b) This scaling is applied at three distinct levels: across tasks, dimensions, and layers.}
    \label{fig:Polarized_Scaling}
\end{figure}

\paragraph{Proposed Polarized Scaling (PS).}
\begin{table*}[t]
\caption{Average absolute accuracy on model merging benchmarks, with normalized average accuracy shown as subscripts in parentheses. ``\textit{Zero-shot}'' (pre-trained model) and ``\textit{Fine-tuned}'' (fine-tuned models) results are presented as the lower and upper bounds, respectively.}
\label{tab:main_results}
\setlength{\tabcolsep}{2.9pt}
\begin{tabular}{lccccccccc}
\toprule[1.5pt] 
\multirow{2}{*}{Method} & \multicolumn{3}{c}{ViT-B/32} & \multicolumn{3}{c}{ViT-B/16} & \multicolumn{3}{c}{ViT-L/14}
\\
\cmidrule[0.5pt](lr){2-4}\cmidrule[0.5pt](lr){5-7}\cmidrule[0.5pt](lr){8-10} & 8 tasks & 14 tasks & 20 tasks & 8 tasks & 14 tasks & 20 tasks & 8 tasks & 14 tasks & 20 tasks
\\
\midrule[0.5pt] 

\rowcolor[HTML]{ECECEC}\textit{Zero-shot} & 48.3 & 57.2 & 56.1 & 55.3 & 61.3 & 59.7 & 64.7 & 68.2 & 65.2 \\
\rowcolor[HTML]{ECECEC}\textit{Fine-tuned} & 92.8 & 90.9 & 91.3 & 94.6 & 92.8 & 93.2 & 95.8 & 94.3 & 94.7 \\

\midrule[0.5pt] 

Weight Averaging \cite{wortsman2022model} & 66.3$_{(72.1)}$ & 64.3$_{(71.1)}$ & 61.0$_{(67.5)}$ & 72.2$_{(76.6)}$ & 69.5$_{(74.8)}$ & 65.3$_{(70.4)}$ & 79.6$_{(83.2)}$ & 76.7$_{(81.1)}$ & 71.6$_{(75.6)}$ \\
Task Arithmetic \cite{ilharcoediting} & 70.8$_{(76.5)}$ & 65.3$_{(72.1)}$ & 60.5$_{(66.8)}$ & 75.4$_{(79.6)}$ & 70.5$_{(75.9)}$ & 65.8$_{(70.8)}$ & 84.9$_{(88.7)}$ & 79.4$_{(84.0)}$ & 74.0$_{(78.1)}$ \\
TIES-Merging \cite{yadav2023ties} & 75.1$_{(81.0)}$ & 68.0$_{(74.8)}$ & 63.4$_{(69.9)}$ & 79.7$_{(84.3)}$ & 73.2$_{(78.7)}$ & 68.2$_{(73.3)}$ & 86.9$_{(90.7)}$ & 79.5$_{(84.1)}$ & 75.7$_{(79.8)}$ \\
Consensus TA \cite{wanglocalizing} & 75.0$_{(80.8)}$ & 70.4$_{(77.4)}$ & 65.4$_{(72.0)}$ & 79.4$_{(83.9)}$ & 74.4$_{(79.9)}$ & 69.8$_{(74.9)}$ & 86.3$_{(90.1)}$ & 82.2$_{(86.9)}$ & 79.0$_{(83.2)}$ \\
TSV-M \cite{gargiulo2025task} & 85.9$_{(92.3)}$ & 80.1$_{(87.9)}$ & 77.1$_{(84.3)}$ & 89.0$_{(93.9)}$ & 84.6$_{(91.0)}$ & 80.6$_{(86.5)}$ & 93.0$_{(97.0)}$ & 89.2$_{(94.4)}$ & 87.7$_{(92.5)}$ \\
Iso-C \cite{marczakno} & 86.3$_{(92.9)}$ & 80.3$_{(88.1)}$ & 75.5$_{(82.5)}$ & 90.6$_{(95.6)}$ & 84.8$_{(91.1)}$ & 79.6$_{(85.4)}$ & 94.2$_{(98.3)}$ & 89.3$_{(94.5)}$ & 87.6$_{(92.2)}$ \\
Iso-CTS \cite{marczakno} & 86.2$_{(92.8)}$ & 81.7$_{(89.7)}$ & 78.1$_{(85.5)}$ & 91.1$_{(96.1)}$ & 86.4$_{(92.8)}$ & 82.4$_{(88.4)}$ & 94.7$_{(98.8)}$ & 91.0$_{(96.3)}$ & 90.1$_{(94.9)}$ \\

\rowcolor[HTML]{C6E2FF}ESM (Ours) & \textbf{88.4}$_{(\textbf{95.3})}$ & \textbf{83.7}$_{(\textbf{92.0})}$ & \textbf{81.3}$_{(\textbf{88.9})}$ & \textbf{91.8}$_{(\textbf{97.0})}$ & \textbf{87.4}$_{(\textbf{94.1})}$ & \textbf{84.9}$_{(\textbf{91.1})}$ & \textbf{94.8}$_{(\textbf{98.9})}$ & \textbf{91.3}$_{(\textbf{96.8})}$ & \textbf{90.4}$_{(\textbf{95.3})}$ \\

\bottomrule[1.5pt] 
\end{tabular}
\end{table*}

Based on these observations, we introduce a simple yet effective Polarized Scaling (PS) scheme designed to proactively suppress the influence of noisy, low-confidence parameters and enhance the strength of essential, high-confidence ones. The core idea is to scale a weight block by the square of its relative magnitude. We apply this scaling at three hierarchical levels:
\begin{enumerate}
    \item \textbf{Inter-Task Scaling.}
    Within the same layer $\ell$, we scale each individual truncated coordinate matrix $\hat{A}_{t}^{(\ell)}$ relative to the average norm across all tasks $T$. This prevents important task signals from being masked by the accumulation of noise from numerous competing tasks.
    The scaling coefficient $s_{t}^{(\ell)}$ for task $t$ at layer $\ell$ is defined as:
    \begin{equation}
    s_{t}^{(\ell)} = \left(\frac{\left| \hat{A}_t^{(\ell)} \right|_F}{\mathbb{E}_{i \in \{1, \dots, T\}}\left[ \left| \hat{A}_{i}^{(\ell)} \right|_F \right]}\right)^2.
    \end{equation}
    The scaled coordinate matrix, $s_{t}^{(\ell)} \cdot \hat{A}_{t}^{(\ell)}$, is used in place of the original $\hat{A}_{t}^{(\ell)}$ to construct $A_{\text{cat}}^{(\ell)}$ in Equation \ref{eq:A_cat}.
    
    \item \textbf{Inter-Dimension Scaling.}
    After task-wise scaling and concatenation, we further scale the columns of the coordinate matrix $A_{\text{cat}}^{(\ell)}$. This highlights the specific dimensions that exhibit strong consensus across all tasks.
    Let $\mathbf{a}_{j}$ be the $j$-th column vector of $A_{\text{cat}}^{(\ell)}$. The scaling coefficient $c_{j}^{(\ell)}$ for column $j$ is calculated as follows:
    \begin{equation}
    c_{j}^{(\ell)} = \left(\frac{\left| \mathbf{a}_{j}^{(\ell)} \right|_2}{\mathbb{E}_{i \in \{1, \dots, d_{\text{in}}\}}\left[ \left| \mathbf{a}_{i}^{(\ell)} \right|_2 \right]}\right)^2.
    \end{equation}
    After applying weighting to each column of $A_{\text{cat}}^{(\ell)}$, the representation becomes more focused on the input dimensions that are effective for every task.
    
    \item \textbf{Inter-Layer Scaling.}
    Finally, we apply scaling across the different layers $\ell \in \mathcal{L}$ of the merged task matrix $\Delta^{(\ell)}_{W_{\text{merged}}}$. This is crucial because residual connections introduce inter-layer competition, where less important layers can mask vital information from important layers.
    We only compare layers of the same nature (e.g., all attention QKV projection layers, or all MLP up-projection layers) as comparing norms across different layer types is meaningless. Let $\mathcal{L}_{\text{type}}$ be the set of layers of the same nature. The scaling factor $\beta_{\ell}$ for layer $\ell \in \mathcal{L}_{\text{type}}$ is:
    \begin{equation}
    \beta_{\ell} = \left(\frac{\left| \Delta_{W_{\text{merged}}}^{(\ell)} \right|_F}{\mathbb{E}_{i \in \mathcal{L}_{\text{type}}}\left[ \left| \Delta_{W_{\text{merged}}}^{(i)} \right|_F \right]}\right)^2.
    \end{equation}
    The scaled merged task matrix is $\beta_{\ell} \cdot \Delta_{W_{\text{merged}}}^{(\ell)}$. This scaling ensures that layers exhibiting stronger directional consensus (higher norm) retain their relative importance.
\end{enumerate}

The parameter for the $\ell$-th layer of the final merged multi-task model is given by:
\begin{equation}
\theta_M^{(\ell)}=\theta_0^{(\ell)}+\alpha \cdot \beta_{\ell} \cdot \Delta_{W_{\text{merged}}}^{(\ell)},
\end{equation}
where $\alpha$ is chosen on a held-out validation set as in \cite{marczakno}.

%% file: sec/4_experiments.tex
\begin{table*}[t]
\caption{Ablation study of the proposed Essential Subspace Decomposition (ESD)-based merging and Polarized Scaling. Results are reported in terms of average absolute accuracy, with normalized average accuracy shown as subscripts in parentheses.}
\label{tab:ablation}
\setlength{\tabcolsep}{4.4pt}
\begin{tabular}{l>{\centering\arraybackslash}p{1cm}>{\centering\arraybackslash}p{1cm}ccc|ccccccc}
\toprule[1.5pt] 
\multirow{2}{*}{Method} & \multicolumn{2}{c}{Decomposition} & \multicolumn{3}{c|}{Polarized Scaling} & \multicolumn{3}{c}{ViT-B/16} & \multicolumn{3}{c}{ViT-L/14}
\\
\cmidrule[0.5pt](lr){2-3}\cmidrule[0.5pt](lr){4-6}\cmidrule[0.5pt](lr){7-9}\cmidrule[0.5pt](lr){10-12} & SVD & ESD & $\ \ \beta_{\ell}\ \ $ & $s_{t}^{(\ell)}$ & $c_{j}^{(\ell)}$ & 8 tasks & 14 tasks & 20 tasks & 8 tasks & 14 tasks & 20 tasks
\\
\midrule[0.5pt] 

Baseline & \textcolor{green}{\ding{51}} & \textcolor{red}{\ding{55}} & \textcolor{red}{\ding{55}} & \textcolor{red}{\ding{55}} & \textcolor{red}{\ding{55}} & 89.0$_{(93.9)}$ & 84.6$_{(91.0)}$ & 80.6$_{(86.5)}$ & 93.0$_{(97.0)}$ & 89.2$_{(94.4)}$ & 87.7$_{(92.5)}$ \\

 & \textcolor{green}{\ding{51}} & \textcolor{red}{\ding{55}} & \textcolor{green}{\ding{51}} & \textcolor{green}{\ding{51}} & \textcolor{green}{\ding{51}} & 89.6$_{(94.6)}$ & 85.4$_{(91.9)}$ & 82.1$_{(88.1)}$ & 93.4$_{(97.4)}$ & 89.6$_{(95.0)}$ & 88.1$_{(93.0)}$ \\

 & \textcolor{red}{\ding{55}} & \textcolor{green}{\ding{51}} & \textcolor{red}{\ding{55}} & \textcolor{red}{\ding{55}} & \textcolor{red}{\ding{55}} & 90.9$_{(95.9)}$ & 85.8$_{(92.2)}$ & 82.8$_{(88.7)}$ & 94.5$_{(98.6)}$ & 90.7$_{(95.9)}$ & 89.5$_{(94.2)}$ \\

 & \textcolor{red}{\ding{55}} & \textcolor{green}{\ding{51}} & \textcolor{green}{\ding{51}} & \textcolor{red}{\ding{55}} & \textcolor{red}{\ding{55}} & 91.4$_{(96.4)}$ & 86.6$_{(93.1)}$ & 83.7$_{(89.6)}$ & 94.6$_{(98.7)}$ & 90.9$_{(96.2)}$ & 90.0$_{(94.8)}$ \\


\rowcolor[HTML]{C6E2FF}ESM (Ours) & \textcolor{red}{\ding{55}} & \textcolor{green}{\ding{51}} & \textcolor{green}{\ding{51}} & \textcolor{green}{\ding{51}} & \textcolor{green}{\ding{51}} & \textbf{91.8}$_{(\textbf{97.0})}$ & \textbf{87.4}$_{(\textbf{94.1})}$ & \textbf{84.9}$_{(\textbf{91.1})}$ & \textbf{94.8}$_{(\textbf{98.9})}$ & \textbf{91.3}$_{(\textbf{96.8})}$ & \textbf{90.4}$_{(\textbf{95.3})}$ \\

\bottomrule[1.5pt] 
\end{tabular}
\end{table*}

\begin{table}[t]
\caption{Time cost of major operations during model merging. 
``\textit{Forward}'': forward pass of 32 proxy samples to obtain $\Delta_O$. ``\textit{PCA}'': principal component analysis performed on $\Delta_O$. ``\textit{Orthogonalization}'': orthogonalization of the matrices $P_{\text{cat}}$ and $A_{\text{cat}}$.
}
\label{tab:overhead}
\setlength{\tabcolsep}{5pt}
\begin{tabular}{lccc}
\toprule[1.5pt] 
Model & \textit{Forward} & \textit{PCA} & \textit{Orthogonalization}
\\
\midrule[0.5pt] 
ViT-B/32 & 0.03 s/task & 1.31 s/task & 13.74 s (once) \\
ViT-B/16 & 0.04 s/task & 1.39 s/task & 13.89 s (once) \\
ViT-L/14 & 0.06 s/task & 5.20 s/task & 70.83 s (once) \\

\bottomrule[1.5pt] 
\end{tabular}
\end{table}

\section{Experiments}
\label{seq:experiments}

\subsection{Experimental Setup}
Following \cite{wanglocalizing}, we evaluate multi-task merging on benchmarks of 8, 14, and 20 tasks. The 8-task set includes: Cars \cite{krause20133d}, DTD \cite{cimpoi2014describing}, EuroSAT \cite{helber2019eurosat}, GTSRB \cite{stallkamp2011german}, MNIST \cite{lecun2002gradient}, RESISC45 \cite{cheng2017remote}, SUN397 \cite{xiao2016sun}, and SVHN \cite{netzer2011reading}. The 14-task benchmark extends this with CIFAR100 \cite{krizhevsky2009learning}, STL10 \cite{coates2011analysis}, Flowers102 \cite{nilsback2008automated}, OxfordIIITPet \cite{parkhi2012cats}, PCAM \cite{veeling2018rotation}, and FER2013 \cite{goodfellow2013challenges}. The 20-task set further incorporates EMNIST \cite{cohen2017emnist}, CIFAR10 \cite{krizhevsky2009learning}, Food101 \cite{bossard2014food}, FashionMNIST \cite{xiao2017fashion}, RenderedSST2 \cite{socher2013recursive}, and KMNIST \cite{clanuwat2018deep}.
We employ three variants of the CLIP model \cite{radford2021learning} with ViT-B/32, ViT-B/16, and ViT-L/14 visual encoders as our pre-trained base models. For the task-specific parameters, we utilize the fine-tuned checkpoints provided in the TALL-masks \cite{wanglocalizing} repository, corresponding to the tasks listed above.
We report performance using both absolute and normalized accuracy, adhering to the standard evaluation practices \cite{wanglocalizing}.

\subsection{Main Results}

As presented in Table \ref{tab:main_results}, we compare our proposed method, ESM, against a comprehensive suite of established model merging methods: Weight Averaging \cite{wortsman2022model}, Task Arithmetic \cite{ilharcoediting}, TIES-Merging \cite{yadav2023ties}, Consensus TA \cite{wanglocalizing}, TSV-M \cite{gargiulo2025task}, and Iso-CTS \cite{marczakno}.
We also include the performance of the zero-shot base model and the average performance of single-task fine-tuned expert models as the lower and upper bounds, respectively.
Our results demonstrate that ESM achieves state-of-the-art performance across all benchmark settings and base model variants. Notably, the performance gains are particularly pronounced in challenging scenarios with a large number of tasks or when using models with smaller capacity (e.g., ViT-B/32 vs. ViT-L/14). This strongly validates the effectiveness of our method in preserving essential task-specific knowledge while significantly mitigating inter-task interference, leading to a superior and more robust merged model.

\subsection{Ablation and Analysis}

\paragraph{Ablation of Proposed ESD and PS.}
We conduct an ablation study of the proposed ESD and PS modules, as shown in Table \ref{tab:ablation}. The results demonstrate that decomposing and merging in the essential subspace leads to a substantial improvement over performing the same operations in the SVD subspace, which is consistent with our theoretical and experimental analyses. Furthermore, the proposed PS module helps reduce task interference while preserving essential task-specific knowledge, leading to additional performance gains. Notably, PS proves beneficial not only when combined with ESD-based merging, but also significantly improves the performance of the baseline method operating in the SVD subspace, highlighting its general applicability.

\paragraph{Computational Overhead.}
We report the major computational overhead introduced by our method, which includes the forward pass of a 32-sample proxy set to compute the activation shift $\Delta_O$, the PCA on $\Delta_O$ to obtain the principal directions for the essential subspace, and the orthogonalization of the matrices $P_{\text{cat}}$ and $A_{\text{cat}}$. As shown in Table \ref{tab:overhead}, the experiments conducted on an RTX 4090 GPU demonstrate that the additional overhead of our method is minimal.

\paragraph{Comparison of ESD and SVD.}
Figure \ref{fig:info_retain} shows the cumulative energy retained as different proportions of components are preserved, defined using the singular values (SVD) or eigenvalues (ESD) as detailed in Appendix \ref{sec:energy_retention}. Our proposed ESD exhibits a highly concentrated energy distribution, indicating its ability to preserve essential task-specific knowledge with fewer components.
Figure \ref{fig:cka} evaluates feature preservation when retaining only 5\% of components per layer, using Centered Kernel Alignment (CKA) similarity \cite{kornblith2019similarity}. We measure the similarity between the merged model and the fine-tuned expert using the class token from the final layer, based on the feature difference relative to the zero-shot pre-trained model. ESD more effectively preserves task-specific features than SVD, further confirming its advantage in retaining critical knowledge.

\begin{figure}[t]
    \centering
    \begin{subfigure}{0.49\linewidth}
        \centering
        \includegraphics[width=\linewidth]{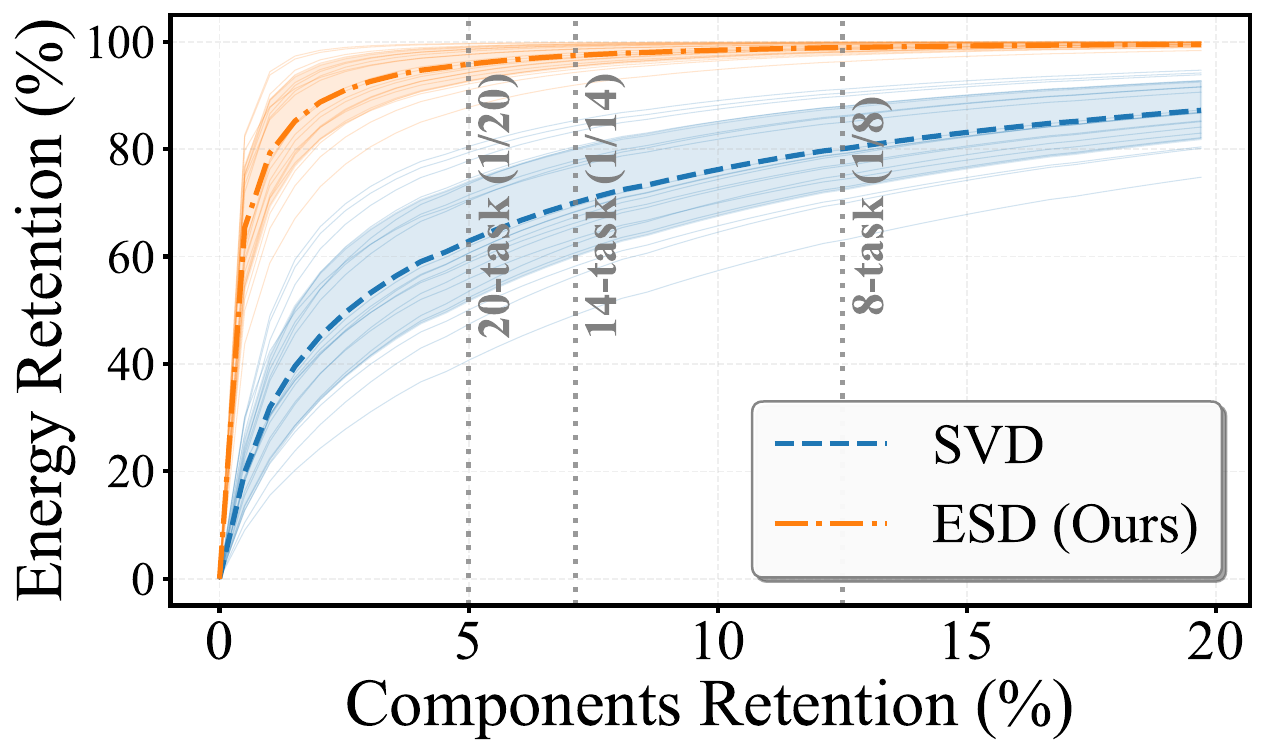}
        \caption{ViT-B/16}
    \end{subfigure}
    \hfill
    \begin{subfigure}{0.49\linewidth}
        \centering
        \includegraphics[width=\linewidth]{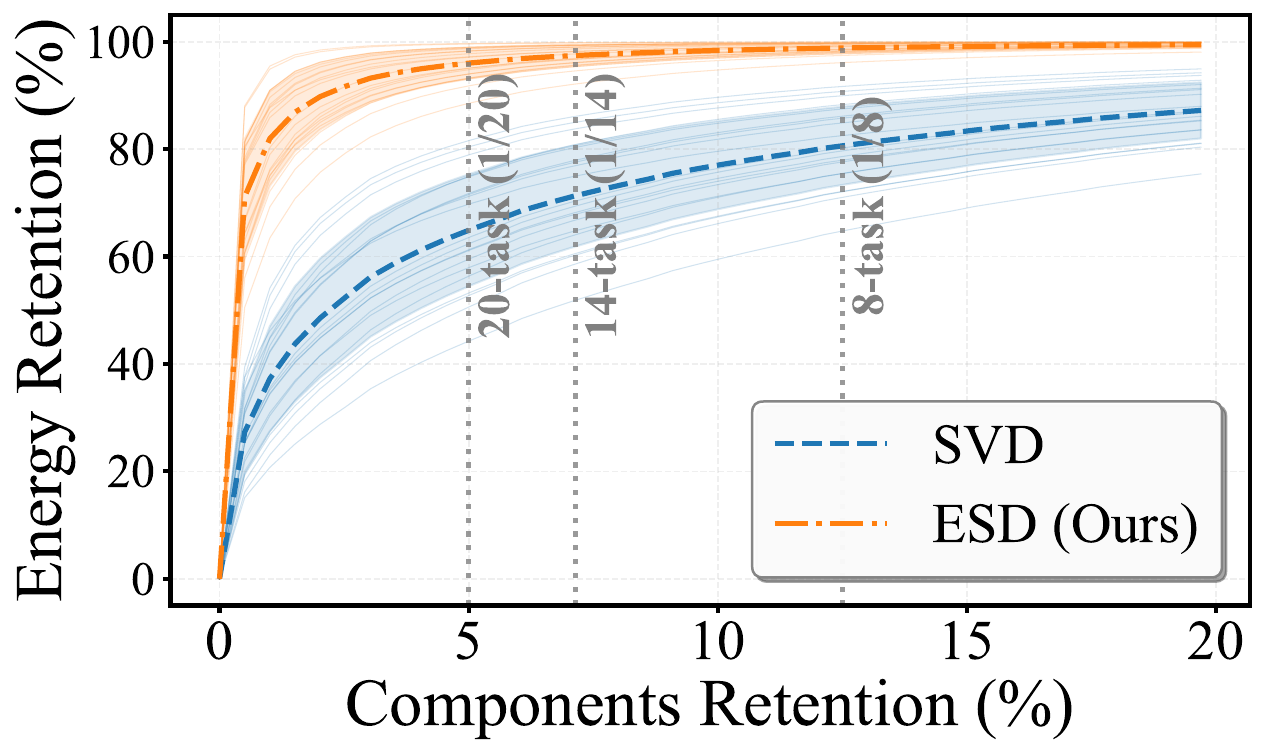}
        \caption{ViT-L/14}
    \end{subfigure}
    \caption{Energy retention as a function of the fraction of principal components retained from the task matrices. Results are reported for each of the 20 tasks as well as the average.}
    \label{fig:info_retain}
\end{figure}

\begin{figure}[t]
    \centering
    \begin{subfigure}{0.49\linewidth}
        \centering
        \includegraphics[width=\linewidth]{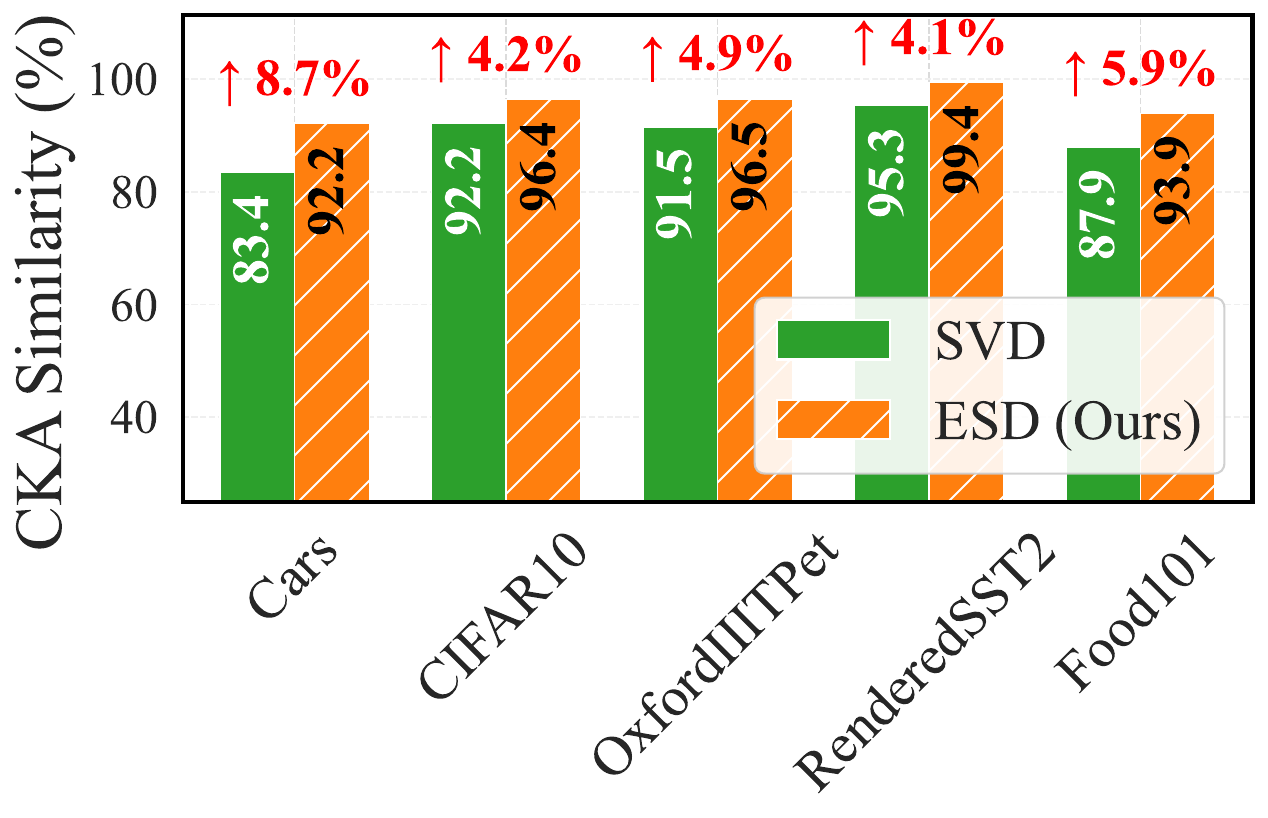}
        \caption{ViT-B/16}
    \end{subfigure}
    \hfill
    \begin{subfigure}{0.49\linewidth}
        \centering
        \includegraphics[width=\linewidth]{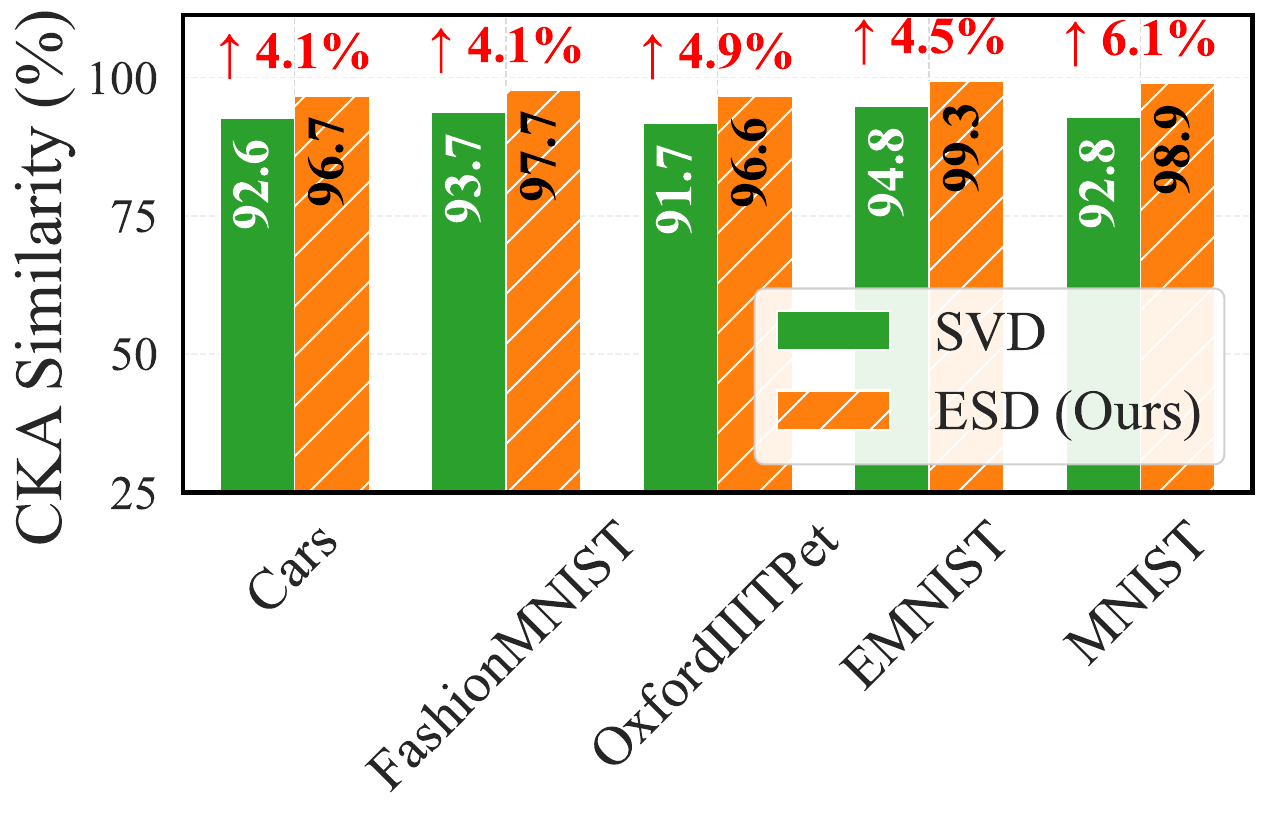}
        \caption{ViT-L/14}
    \end{subfigure}
    \caption{CKA similarity with the fine-tuned model when retaining only 5\% of principal components using SVD and ESD.
    }
    \label{fig:cka}
\end{figure}

\paragraph{Impact of Proxy Dataset Composition.}
We analyze how the composition of the proxy dataset affects Essential Subspace Merging (ESM). Table \ref{tab:proxy_composition} shows the average performance and variance over five runs, with a fixed proxy set size of 32 samples.
Our default setup uses unlabeled samples randomly selected from the respective task's dataset, denoted as ``\textit{Random (ID)}''. We also experiment with extreme scenarios: using samples from only a single class within the task data (``\textit{Class Imbalance}'') and using samples randomly drawn from an out-of-distribution dataset (ImageNet-1k \cite{russakovsky2015imagenet}), denoted as ``\textit{Random (OOD)}''.
Results indicate stable merging performance even with large distribution shifts, which we attribute to the consistent sparse patterns in the features extracted by each task-specific model, regardless of the input data distribution.
This finding enhances the generality of our approach, demonstrating that it provides stable performance gains over SVD-based merging methods even when task-specific data is unavailable.

\begin{table}[t]
\caption{Impact of proxy dataset composition on merging performance, evaluated on the 8-task benchmark. ``\textit{Random (ID)}'': random sampling from in-distribution task data. ``\textit{Class Imbalance}'': sampling only a single class per task. ``\textit{Random (OOD)}'': random sampling from out-of-distribution ImageNet-1k \cite{russakovsky2015imagenet}.}
\label{tab:proxy_composition}
\setlength{\tabcolsep}{2.3pt}
\begin{tabular}{lc|ccc}
\toprule[1.5pt] 
Method & Sampling & ViT-B/32 & ViT-B/16 & ViT-L/14
\\
\midrule[0.5pt] 

\rowcolor[HTML]{ECECEC}SVD & - & 86.6 & 89.6 & 93.4 \\
\rowcolor[HTML]{C6E2FF}ESD & \textit{Random (ID)} & 88.4$_{\pm 0.06}$ & 91.8$_{\pm 0.04}$ & 94.8$_{\pm 0.07}$ \\
ESD & \textit{Class Imbalance} & 88.4$_{\pm 0.10}$ & 91.8$_{\pm 0.06}$ & 94.8$_{\pm 0.04}$ \\
ESD & \textit{Random (OOD)} & 88.3$_{\pm 0.04}$ & 91.8$_{\pm 0.04}$ & 94.8$_{\pm 0.01}$ \\

\bottomrule[1.5pt] 
\end{tabular}
\end{table}

\begin{figure}[t]
    \centering
    \begin{subfigure}{0.49\linewidth}
        \centering
        \includegraphics[width=\linewidth]{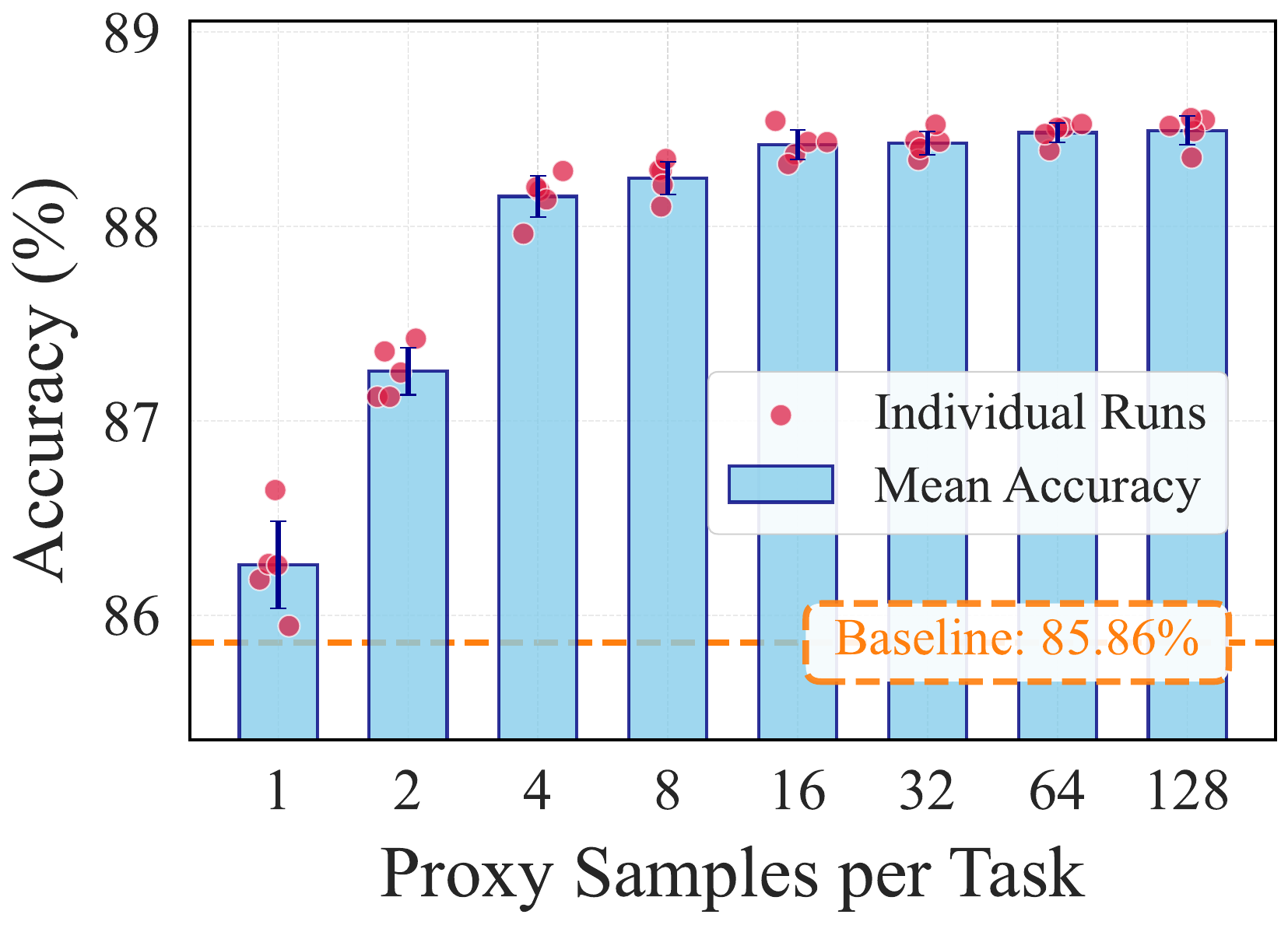}
        \caption{ViT-B/32}
    \end{subfigure}
    \hfill
    \begin{subfigure}{0.49\linewidth}
        \centering
        \includegraphics[width=\linewidth]{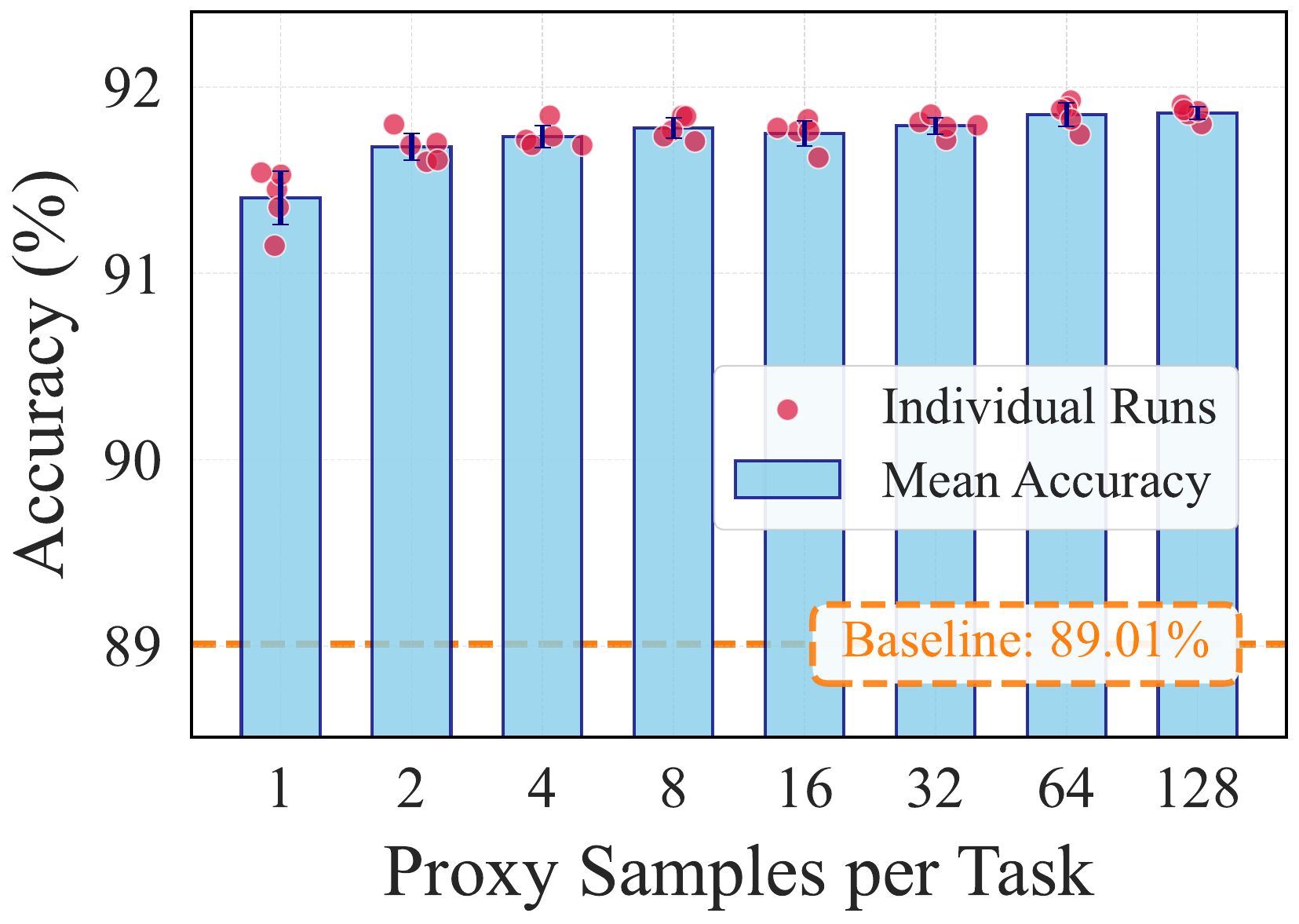}
        \caption{ViT-B/16}
    \end{subfigure}
    \caption{Ablation study on the number of proxy samples used in ESM, reporting the average accuracy on the 8-task benchmark. Each configuration was evaluated over five random runs.}
    \label{fig:ablation_proxy_size}
\end{figure}

\paragraph{Ablation of Proxy Dataset Size.}
We perform an ablation study on the size of the proxy dataset used in Essential Subspace Merging (ESM), as illustrated in Figure~\ref{fig:ablation_proxy_size}. The results show that using only four unlabeled samples per task is sufficient to achieve robust performance, while even a single sample can consistently outperform the baseline.

\paragraph{Detailed Ablation Study of Polarized Scaling.}
We perform a detailed ablation of Polarized Scaling (PS) in Table~\ref{tab:ablation_ps}. We compare three alternatives: (i) ``\textit{Reverse}'', which applies the reciprocal of the scaling factors; (ii) ``\textit{Noise$--$}'', which retains only factors $<1$ to suppress noisy parameters; and (iii) ``\textit{Signal$++$}'', which retains only factors $>1$ to enhance important parameters.
Experimental results show that, compared with ``\textit{None}'' (i.e., without scaling), the ``\textit{Reverse}'' operation significantly degrades performance because important parameters are overwhelmed by redundant ones. Both ``\textit{Noise$--$}'' and ``\textit{Signal$++$}'' improve over ``\textit{None}'' by raising the signal-to-noise ratio of important parameters. Our full PS method, which combines both suppression and amplification, achieves the best performance.

\begin{table}[t]
\caption{Detailed ablation study of the proposed Polarized Scaling. The symbol $\gamma$ denotes the scaling factor at three different levels: $s$, $c$, or $\beta$. The following variants are compared: (i) ``\textit{Reverse}'': taking the reciprocal of the scaling factors; (ii) ``\textit{Noise$--$}'': retaining only factors $<1$ to suppress noisy parameters; (iii) ``\textit{Signal$++$}'': retaining only factors $>1$ to enhance important parameters.}
\label{tab:ablation_ps}
\setlength{\tabcolsep}{3.8pt}
\begin{tabular}{lcccc}
\toprule[1.5pt] 
\multirow{2}{*}{Method} & \multirow{2}{*}{Scaling} & \multicolumn{3}{c}{ViT-B/32}
\\
\cmidrule[0.5pt](lr){3-5} && 8 tasks & 14 tasks & 20 tasks
\\
\midrule[0.5pt] 

\rowcolor[HTML]{ECECEC}\textit{None} & - & 86.7$_{(93.4)}$ & 81.1$_{(89.0)}$ & 78.1$_{(85.5)}$ \\
\textit{Reverse} & $1/\gamma$ & 82.9$_{(89.2)}$ & 76.3$_{(83.7)}$ & 72.6$_{(79.4)}$ \\
\textit{Noise$--$} & $\min(\gamma,1)$ & 87.4$_{(94.1)}$ & 83.0$_{(91.2)}$ & 80.5$_{(88.1)}$ \\
\textit{Signal$++$} & $\max(\gamma,1)$ & 87.7$_{(94.6)}$ & 82.8$_{(91.0)}$ & 80.2$_{(87.8)}$ \\

\rowcolor[HTML]{C6E2FF}PS (Ours) &  $\gamma$ & \textbf{88.4}$_{(\textbf{95.3})}$ & \textbf{83.7}$_{(\textbf{92.0})}$ & \textbf{81.3}$_{(\textbf{88.9})}$ \\

\bottomrule[1.5pt] 
\end{tabular}
\end{table}

\begin{figure}[t]
    \centering
    \begin{subfigure}{0.49\linewidth}
        \centering
        \includegraphics[width=\linewidth]{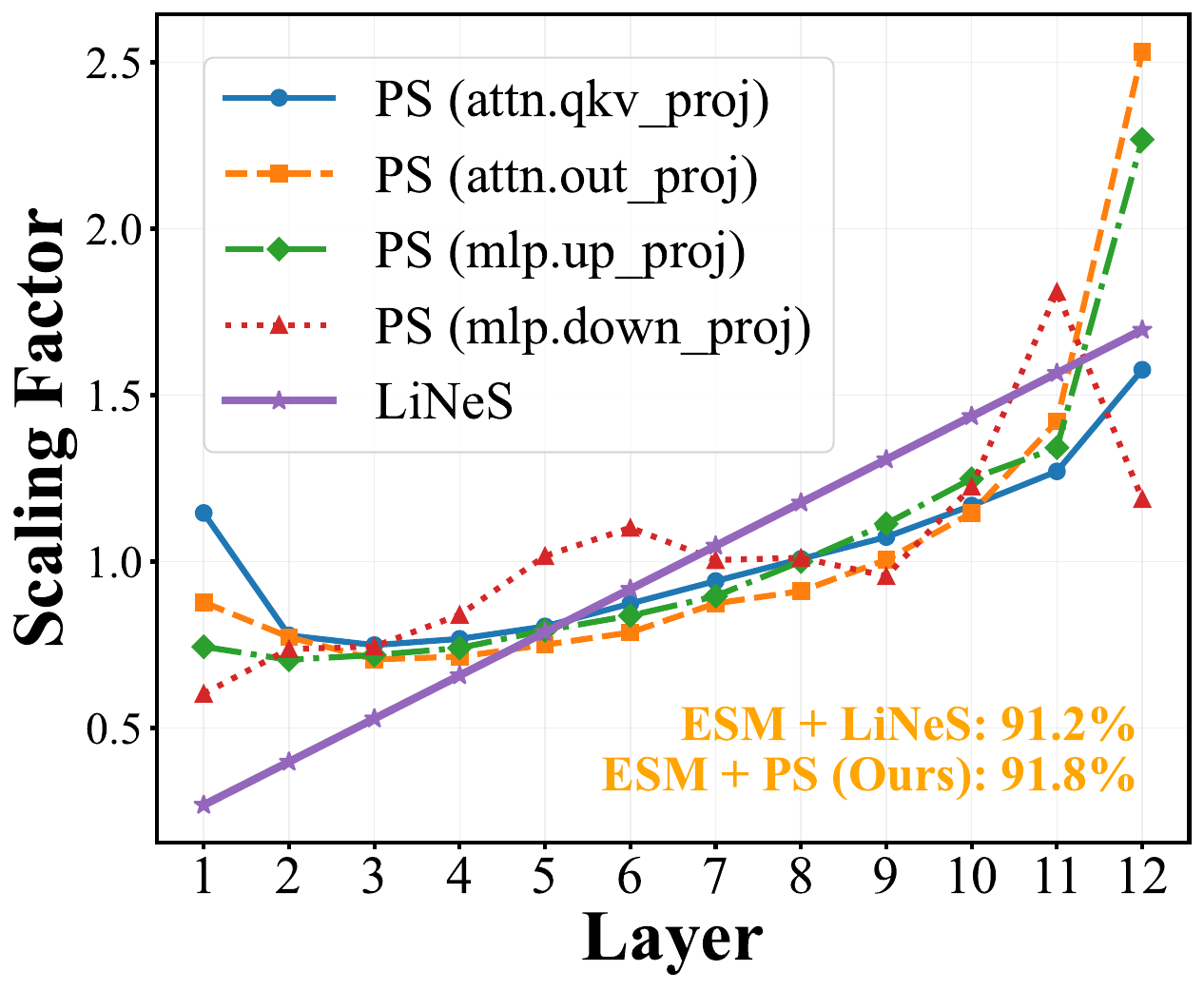}
        \caption{8-task}
    \end{subfigure}
    \hfill
    \begin{subfigure}{0.49\linewidth}
        \centering
        \includegraphics[width=\linewidth]{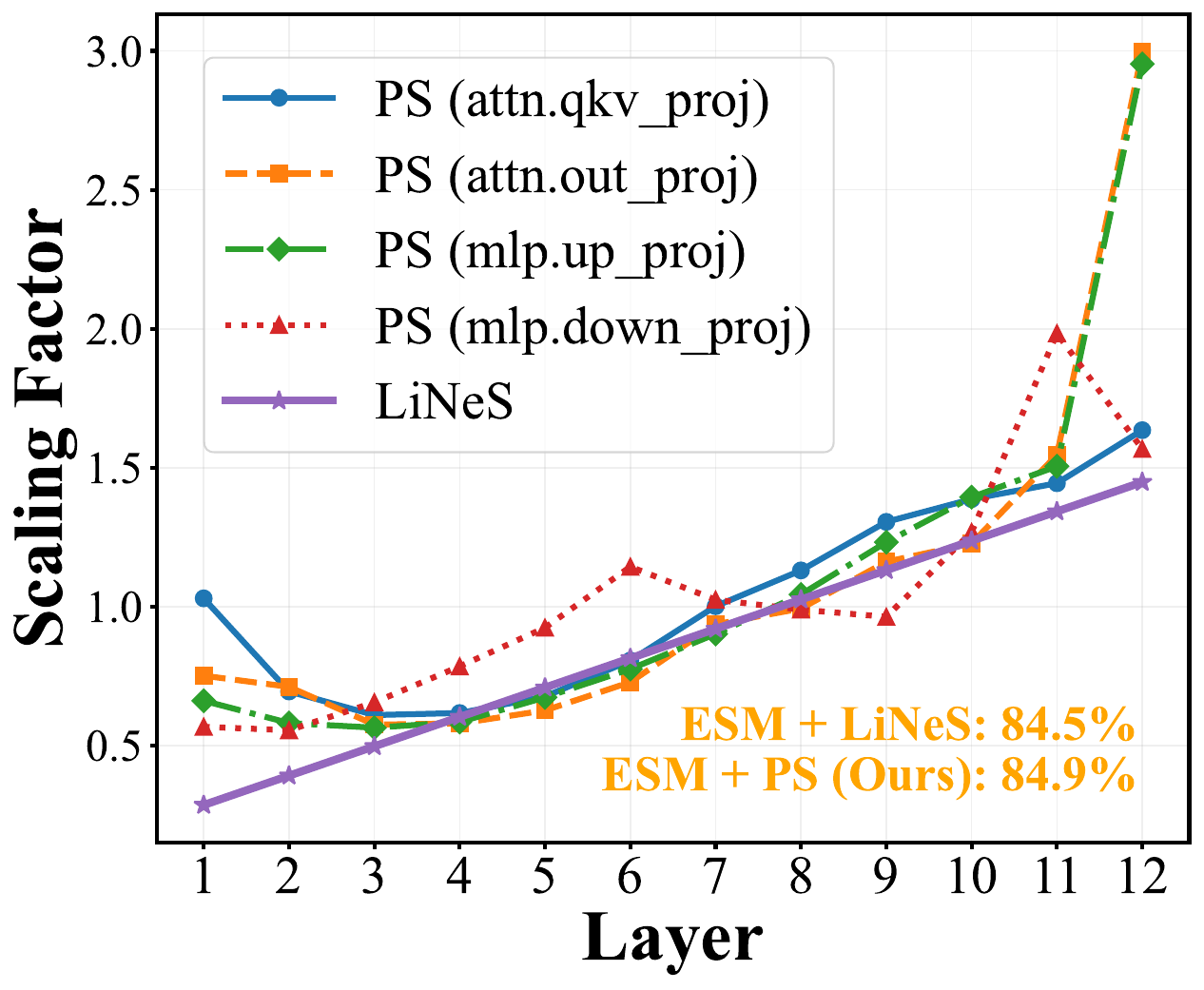}
        \caption{20-task}
    \end{subfigure}
    \caption{Layer-wise scaling coefficients for merged ViT-B/16.}
    \label{fig:PS_vs_LiNeS}
\end{figure}

\paragraph{Analysis of Scaling Coefficients.}
Figure \ref{fig:PS_vs_LiNeS} compares the layer-wise coefficients of our Polarized Scaling (PS) with LiNeS \cite{wanglines}. Based on the functional roles of shallow and deep blocks, LiNeS applies coefficients that increase linearly with depth, requiring a validation set to determine the linear multiplier. In contrast, our PS method provides a more refined measure of parameter importance by leveraging layer norms. It performs separate scaling for four key layer types in a ViT model: the QKV and Output projections in the attention module, and the Up and Down projections in the MLP. This approach not only achieves superior merging performance but also removes the need for a validation set by computing coefficients directly from parameter norms.

%% file: sec/5_conclusion.tex
\section{Conclusion}
In this paper, we propose Essential Subspace Merging (ESM), a novel model merging framework that narrows the gap to expert models. It introduces the ESD, a feature shift distribution-aware decomposition method that identifies a sparser and more functionally critical subspace than SVD, thereby reducing inter-task interference. Observing that high-norm updates indicate consensus-driven directions, we further propose Polarized Scaling (PS), a three-level mechanism that amplifies high-confidence signals while suppressing noise. ESM achieves state-of-the-art performance and provides a robust and effective solution for model merging.

\section*{Acknowledgment}
This research was supported by the Jiangsu Science Foundation (BG2024036, BK20243012), the National Science Foundation of China (62125602, U24A20324, 92464301), the New Cornerstone Science Foundation through the XPLORER PRIZE, the Fundamental Research Funds for the Central Universities (2242025K30024), and the Big Data Computing Center of Southeast University.

%% file: sec/X_suppl.tex
\clearpage
\setcounter{page}{1}
\maketitlesupplementary

\appendix
\newtheorem{theorem}{Theorem}

\section{Proofs}

This section provides the proofs for the expected output error after truncation for both the standard SVD and the proposed Essential Subspace Decomposition (ESD), as well as the equivalence of whitening and Orthogonal Procrustes.

\subsection{Proof for SVD Truncation Error}
\label{sec:Proof_SVD_Truncation_Error}

\begin{theorem}
Given a task matrix $\Delta_W \in \mathbb{R}^{d_{\text{out}} \times d_{\text{in}}}$ with its singular value decomposition $\Delta_W = U\Sigma V^\top = \sum_{i=1}^r \sigma_i u_i v_i^\top$. Let $\widehat{\Delta_W} = \sum_{i=1}^k \sigma_i u_i v_i^\top$ be its top-$k$ rank approximation. For an input $x$ drawn from a distribution $\mathcal{D}$, the expected squared $L_2$ error on the output activation is:
$$ \mathbb{E}_{x\sim\mathcal{D}} \left[ \|\Delta_W x - \widehat{\Delta_W}x \|_2^2 \right] = \sum_{i=k+1}^r\sigma_{i}^{2} \cdot \mathbb{E}_{x\sim\mathcal{D}} \left[ (v_{i}^{\top}x)^2 \right].
$$
\end{theorem}

\begin{proof}
The error matrix resulting from the truncation is the sum of the discarded components:
\begin{equation}
\Delta_W - \widehat{\Delta_W} = \sum_{i=k+1}^{r} \sigma_i u_i v_i^\top.
\end{equation}
The error on the output activation for a given input $x$ is:
\begin{equation}
(\Delta_W - \widehat{\Delta_W})x = \left( \sum_{i=k+1}^{r} \sigma_i u_i v_i^\top \right) x = \sum_{i=k+1}^{r} \sigma_i u_i (v_i^\top x).
\end{equation}
Since $v_i^\top x$ is a scalar, we can rewrite this as a linear combination of the orthonormal vectors $u_i$. The squared $L_2$ norm of this error vector is:
\begin{equation}
\|(\Delta_W - \widehat{\Delta_W})x\|_2^2 = \left\| \sum_{i=k+1}^{r} (\sigma_i v_i^\top x) u_i \right\|_2^2.
\end{equation}
Because the left singular vectors $\{u_i\}$ form an orthonormal set, the squared norm of their weighted sum is the sum of the squares of the weights:
\begin{equation}
\|(\Delta_W - \widehat{\Delta_W})x\|_2^2 = \sum_{i=k+1}^{r} (\sigma_i v_i^\top x)^2 = \sum_{i=k+1}^{r} \sigma_i^2 (v_i^\top x)^2.
\end{equation}
By taking the expectation over the input distribution $\mathcal{D}$ and applying the linearity of expectation, we arrive at the final expression:
\begin{equation}
\mathbb{E}_{x\sim\mathcal{D}} \left[ \|\Delta_W x - \widehat{\Delta_W}x \|_2^2 \right] = \sum_{i=k+1}^r\sigma_{i}^{2} \cdot \mathbb{E}_{x\sim\mathcal{D}} \left[ (v_{i}^{\top}x)^2 \right].
\end{equation}
This completes the proof.
\end{proof}

\subsection{Proof for ESD Truncation Error}
\label{sec:Proof_ESD_Truncation_Error}

\begin{theorem}
Given a task matrix $\Delta_W \in \mathbb{R}^{d_{\text{out}} \times d_{\text{in}}}$, let $\hat{P} \in \mathbb{R}^{d_{\text{out}} \times k}$ be the matrix whose columns are the top-$k$ principal components (eigenvectors) derived from the activation shift matrix $\Delta_O = X_{\text{proxy}}\Delta_W^\top$. Let $\widehat{\Delta_W} = \hat{P}\hat{A} = \hat{P} (\hat{P}^\top \Delta_W)$ be the ESD reconstruction. For an input $x \sim \mathcal{D}$, the expected squared $L_2$ error on the output is proportional to the sum of the discarded eigenvalues:
$$ \mathbb{E}_{x\sim\mathcal{D}} \left[ \|\Delta_W x - \widehat{\Delta_W}x \|_2^2 \right] = \sum_{i=k+1}^{d_{\text{out}}} \lambda_i. $$
\end{theorem}

\begin{proof}
The error on the output activation for an input $x$ is:
\begin{equation}
\Delta_W x - \widehat{\Delta_W}x = \Delta_W x - \hat{P} \hat{P}^\top \Delta_W x = (I - \hat{P} \hat{P}^\top) \Delta_W x.
\end{equation}
The matrix $(I - \hat{P} \hat{P}^\top)$ is the projection matrix onto the subspace spanned by the discarded eigenvectors $\{p_{k+1}, \dots, p_{d_{\text{out}}}\}$. Let $y = \Delta_W x$ be the activation shift for input $x$. The error vector can be expressed as the projection of $y$ onto this orthogonal subspace:
\begin{equation}
(I - \hat{P} \hat{P}^\top) y = \sum_{i=k+1}^{d_{\text{out}}} (p_i^\top y) p_i.
\end{equation}
Since $\{p_i\}$ form an orthonormal basis, the squared $L_2$ norm is the sum of the squares of the projection coefficients:
\begin{equation}
\begin{aligned}
\|\Delta_W x - \widehat{\Delta_W}x \|_2^2 &= \left\| \sum_{i=k+1}^{d_{\text{out}}} (p_i^\top y) p_i \right\|_2^2 \\
&= \sum_{i=k+1}^{d_{\text{out}}} (p_i^\top y)^2 \\
&= \sum_{i=k+1}^{d_{\text{out}}} (p_i^\top \Delta_W x)^2.
\end{aligned}
\end{equation}
Now, we take the expectation over the input distribution $\mathcal{D}$:
\begin{equation}
\begin{aligned}
\mathbb{E}_{x\sim\mathcal{D}} \left[ \|\Delta_W x - \widehat{\Delta_W}x \|_2^2 \right] &= \mathbb{E}_{x\sim\mathcal{D}} \left[ \sum_{i=k+1}^{d_{\text{out}}} (p_i^\top \Delta_W x)^2 \right] \\
&= \sum_{i=k+1}^{d_{\text{out}}} \mathbb{E}_{x\sim\mathcal{D}} \left[ (p_i^\top \Delta_W x)^2 \right].
\end{aligned}
\end{equation}
By the definition of Principal Component Analysis (PCA), the eigenvalue $\lambda_i$ of the covariance matrix of activation shifts corresponds to the variance of the activation shifts projected onto the $i$-th principal component $p_i$. Assuming the activation shifts are centered (or mean-subtracted during PCA), this variance is:
\begin{equation}
\begin{aligned}
\lambda_i &= \text{Var}(p_i^\top \Delta_W x)\\
&= \mathbb{E}_{x\sim\mathcal{D}} \left[ (p_i^\top \Delta_W x)^2 \right] - \left( \mathbb{E}_{x\sim\mathcal{D}} [p_i^\top \Delta_W x] \right)^2 \\
&= \mathbb{E}_{x\sim\mathcal{D}} \left[ (p_i^\top \Delta_W x)^2 \right].
\end{aligned}
\end{equation}
Substituting this result back into our error expression, we get:
\begin{equation}
\mathbb{E}_{x\sim\mathcal{D}} \left[ \|\Delta_W x - \widehat{\Delta_W}x \|_2^2 \right] = \sum_{i=k+1}^{d_{\text{out}}} \lambda_i.
\end{equation}
This completes the proof, showing that the expected error is solely dependent on the magnitude of the discarded eigenvalues, which represent the functional variance captured by those directions.
\end{proof}

\subsection{Equivalence of Whitening and Orthogonal Procrustes}
\begin{theorem}
The transformations \( X \mapsto X (X^{\top} X)^{-1/2} \) (whitening) and \( X \mapsto U V^{\top} \) (Orthogonal Procrustes), where \( X = U \Sigma V^{\top} \) is the singular value decomposition (SVD) of \( X \), are equivalent.
\end{theorem}

\begin{proof}
Let \( X = U \Sigma V^{\top} \) be the SVD of \( X \), where \( \Sigma \) is a diagonal matrix of singular values, we have:
\begin{equation}
X^{\top} X = V \Sigma^{2} V^{\top}.
\end{equation}
It follows that:
\begin{equation}
(X^{\top} X)^{-1/2} = V \Sigma^{-1} V^{\top}.
\end{equation}
Substituting this into the whitening transformation gives:
\begin{equation}
\begin{aligned}
X (X^{\top} X)^{-1/2} &= (U \Sigma V^{\top})(V \Sigma^{-1} V^{\top})\\
&= U \Sigma \Sigma^{-1} V^{\top}\\
&= U V^{\top}.
\end{aligned}
\end{equation}
The whitening operation is equivalent to solving the Orthogonal Procrustes problem, completing the proof.
\end{proof}

\section{Method Details}
\subsection{Methodology Pseudocode}
The pseudocode of our proposed model merging approach is presented in Algorithm \ref{alg:esm}. Steps highlighted in \textcolor{orange}{\textit{orange}} correspond to the essential subspace merging, while \textcolor{blue}{\textit{blue}} annotations indicate the polarized scaling. Our method decomposes each task matrix within its essential subspace and performs truncation, then concatenates the components from all task matrices to reconstruct a new merged matrix. Through three-level scaling, we ultimately obtain a fused model that effectively preserves essential task knowledge while maintaining minimal inter-task interference.

\begin{algorithm}[t]
\caption{Essential Subspace Merging}
\label{alg:esm}
\begin{algorithmic}[1]
\Require Task matrices $\{\Delta_{W_t}^{(\ell)}\}_{t=1}^T$ for all layers $\ell \in \mathcal{L}$, pre-trained weights $\theta_0$, validation set $\mathcal{D}_{\text{val}}$
\Ensure Merged model parameters $\theta_M$

\noindent\Comment{\textcolor{orange}{-------------------------- \textit{Decomposition and Truncation}}}
\For{each task $t = 1$ to $T$ and each layer $\ell$}
    \State Obtain essential basis $P_t^{(\ell)}$
    \State Compute coordinate matrix: $A_t^{(\ell)} \gets (P_t^{(\ell)})^\top \Delta_{W_t}^{(\ell)}$
    \State Truncate to top-$k$ components: $\hat{P}_t^{(\ell)} \gets P_t^{(\ell)}[:,1:k]$, $\hat{A}_t^{(\ell)} \gets A_t^{(\ell)}[1:k,:]$, where $k = \lfloor d_{\text{out}}/T \rfloor$
\EndFor

\noindent\Comment{\textcolor{orange}{--------------------------------------------- \textit{Concatenation}}}
\For{each layer $\ell$}
    \State $P_{\text{cat}}^{(\ell)} \gets [\hat{P}_1^{(\ell)}, \hat{P}_2^{(\ell)}, \ldots, \hat{P}_T^{(\ell)}]$
    \State $A_{\text{cat}}^{(\ell)} \gets [\hat{A}_1^{(\ell)}; \hat{A}_2^{(\ell)}; \ldots; \hat{A}_T^{(\ell)}]$
    
    \Comment{\textcolor{blue}{----------------------------------- \textit{Inter-Task Scaling}}}
    \For{each task $t = 1$ to $T$}
        \State $\hat{A}_t^{(\ell)} \gets s_t^{(\ell)} \cdot \hat{A}_t^{(\ell)}$
    \EndFor
    
    \Comment{\textcolor{blue}{--------------------------- \textit{Inter-Dimension Scaling}}}
    \For{each column $j = 1$ to $d_{\text{in}}$}
        \State $a_j^{(\ell)} \gets c_j^{(\ell)} \cdot a_j^{(\ell)}$
    \EndFor
\EndFor

\noindent\Comment{\textcolor{orange}{----------------- \textit{Orthogonalization and Reconstruction}}}
\For{each layer $\ell$}
    \State Compute SVD: $P_{\text{cat}}^{(\ell)} = U_P^{(\ell)} \Sigma_P^{(\ell)} (V_P^{(\ell)})^\top$
    \State Compute SVD: $A_{\text{cat}}^{(\ell)} = U_A^{(\ell)} \Sigma_A^{(\ell)} (V_A^{(\ell)})^\top$
    \State Orthogonalize: $\tilde{P}^{(\ell)} \gets U_P^{(\ell)} (V_P^{(\ell)})^\top$, $\tilde{A}^{(\ell)} \gets U_A^{(\ell)} (V_A^{(\ell)})^\top$
    \State Reconstruct: $\Delta_{\text{merged}}^{(\ell)} \gets \tilde{P}^{(\ell)} \tilde{A}^{(\ell)}$
\EndFor

\noindent\Comment{\textcolor{blue}{--------------------------------------- \textit{Inter-Layer Scaling}}}
\For{each layer $\ell \in \mathcal{L}$}
    \State $\Delta_{\text{merged}}^{(\ell)} \gets \beta_\ell \cdot \Delta_{\text{merged}}^{(\ell)}$
\EndFor

\State Find optimal $\alpha^*$ on $\mathcal{D}_{\text{val}}$ that maximizes performance
\State $\theta_M \gets \theta_0 + \alpha^* \cdot \{\Delta_{\text{merged}}^{(\ell)}\}_{\ell=1}^L$

\State \Return $\theta_M$
\end{algorithmic}
\end{algorithm}

\subsection{Details on Polarized Scaling}
In the empirical evidence presented in Section \ref{sec:polarized_scaling}, individual or merged task matrices are loaded in different sequences. To isolate the phenomenon from the distinct characteristics of different layer types, our analysis focuses solely on the four principal layer types within the ViT: namely, the QKV projection and the Output projection in the attention module, along with the Up- and Down-projections in the MLP. Furthermore, when merging layers sequentially, we group the layers by these four layer types. The merging process then proceeds by cyclically selecting task matrices from each group in a fixed order. This design ensures that the observed effects are attributable to the merging strategy rather than to variations across layer types.

\paragraph{Empirical Evidence: Pairwise Task Interaction.}
We further analyze the pairwise influence between task matrices. As shown in Figure \ref{fig:task_interaction}, each column represents how the performance of two tasks changes when the task matrix of the column’s task is added in different orders to the fine-tuned model of the row’s task. The results show that adding task matrices in descending order of their norms yields a better Pareto optimal trade-off than adding them in ascending order. Introducing a large number of low-norm updates brings little improvement to the target task but significantly harms the performance of others. This highlights the importance of suppressing less critical or noisy updates while emphasizing the most essential ones in model merging.

\begin{figure}[t]
    \centering
    \includegraphics[width=\linewidth]{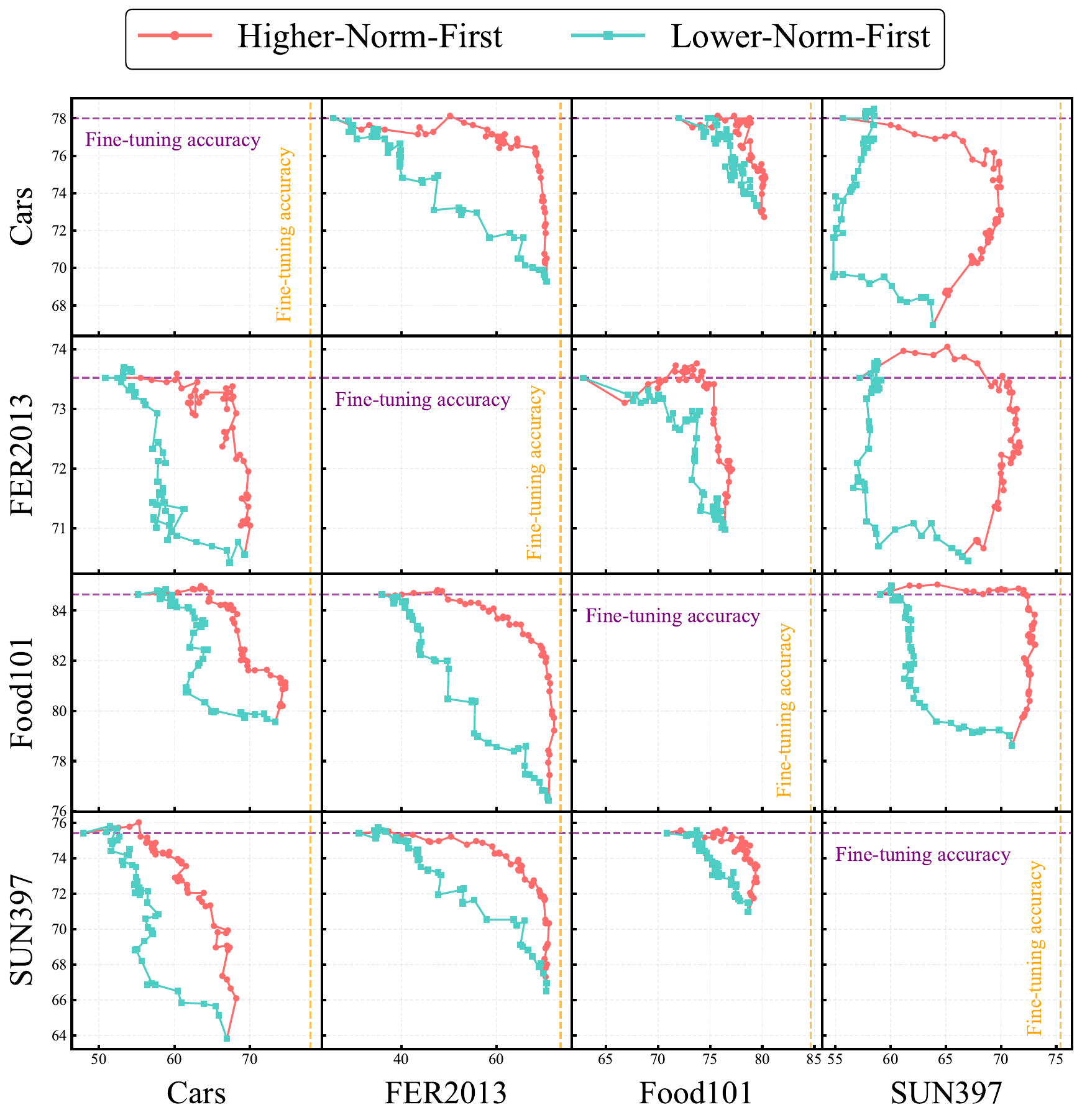}
    \caption{Illustration of task invasion: using the task matrices from one task to invade the fine-tuned model of another, performed layer-by-layer based on the norm order. Rows represent the invaded task, and columns represent the invading task.}
    \label{fig:task_interaction}
\end{figure}

\subsection{Merging Non-Matrix Parameters}
While most parameters in the ViT are 2D matrices merged using our proposed ESM within the Essential Subspace, the network also includes other parameter types. For non-matrix parameters such as bias vectors, layer normalization parameters, and the convolutional stem, we follow the standard practice in \cite{gargiulo2025task} and apply simple averaging.

\section{Experiment Details}

\subsection{OOD Performance}
We evaluate the out-of-domain (OOD) generalization of our model merging method and that of the baseline \cite{gargiulo2025task}, as presented in Table \ref{tab:ood_performance}. Based on the 8-task benchmark, we treat each task in turn as the test task, merge the fine-tuned models from the remaining seven tasks, and evaluate the merged model on the held-out task. The final OOD score is the average performance across all test tasks. The results show that our method not only achieves substantial improvements in conventional in-domain evaluations but also demonstrates superior generalization in the out-of-domain setting.

\begin{table}[t]
\caption{Performance on the 8-task benchmark. ``ID'': the average accuracy of the merged model across all tasks. ``OOD'': the result of merging fine-tuned models from 7 tasks and testing on the remaining unseen task. Each task is alternately used as the test task, and the reported OOD score is the average across all such settings.}
\label{tab:ood_performance}
\setlength{\tabcolsep}{5pt}
\begin{tabular}{lcccccc}
\toprule[1.5pt] 
\multirow{2}{*}{Method} & \multicolumn{2}{c}{ViT-B/32} & \multicolumn{2}{c}{ViT-B/16} & \multicolumn{2}{c}{ViT-L/14}
\\
\cmidrule[0.5pt](lr){2-3} \cmidrule[0.5pt](lr){4-5} \cmidrule[0.5pt](lr){6-7} & ID & OOD & ID & OOD & ID & OOD 
\\
\midrule[0.5pt] 

\rowcolor[HTML]{ECECEC}Baseline & 85.9 & 49.9 & 89.0 & 54.7 & 93.0 & 65.9 \\
\rowcolor[HTML]{C6E2FF}ESM (Ours) & \textbf{88.4} & \textbf{51.3} & \textbf{91.8} & \textbf{55.5} & \textbf{94.8} & \textbf{66.4}  \\

\bottomrule[1.5pt] 
\end{tabular}
\end{table}

\subsection{Selection of the Global Scaling Coefficient \texorpdfstring{$\boldsymbol{\alpha}$}{}}
We report the global scaling coefficient $\alpha$ selected on the validation set, as shown in Table \ref{tab:alpha_selection}. Based on the empirical ranges used in previous model merging studies \cite{gargiulo2025task, marczakno}, we set the search interval for $\alpha$ between 0.0 and 2.0 and perform binary search to determine the optimal value. The results show that the optimal $\alpha$ decreases as the number of tasks increases, likely because merging more tasks amplifies the norm of the combined updates. Consequently, a smaller scaling factor helps balance the output feature norm and achieves better validation performance.

\begin{table}[t]
\caption{Scaling coefficient $\alpha$ selected using the validation set. $T$ represents the total number of merged task-specific expert models.}
\label{tab:alpha_selection}
\setlength{\tabcolsep}{2.55pt}
\begin{tabular}{ccccccccc}
\toprule[1.5pt] 
\multicolumn{3}{c}{ViT-B/32} & \multicolumn{3}{c}{ViT-B/16} & \multicolumn{3}{c}{ViT-L/14}
\\
\cmidrule[0.5pt](lr){1-3} \cmidrule[0.5pt](lr){4-6} \cmidrule[0.5pt](lr){7-9} $T$=8 & $T$=14 & $T$=20 & $T$=8 & $T$=14 & $T$=20 & $T$=8 & $T$=14 & $T$=20
\\
\midrule[0.5pt] 

0.88 & 0.76 & 0.69 & 0.82 & 0.76 & 0.67 & 0.91 & 0.72 & 0.64 \\

\bottomrule[1.5pt] 
\end{tabular}
\end{table}

\subsection{Performance on Individual Tasks}
We present detailed model fusion results across multiple models and datasets, as shown in Figure \ref{fig:radar_comparison}. Our proposed ESM framework demonstrates consistent performance improvements across nearly all datasets and models.

\begin{figure*}[t]
    \centering

    \begin{subfigure}{\linewidth}
        \centering
        \includegraphics[width=\linewidth]{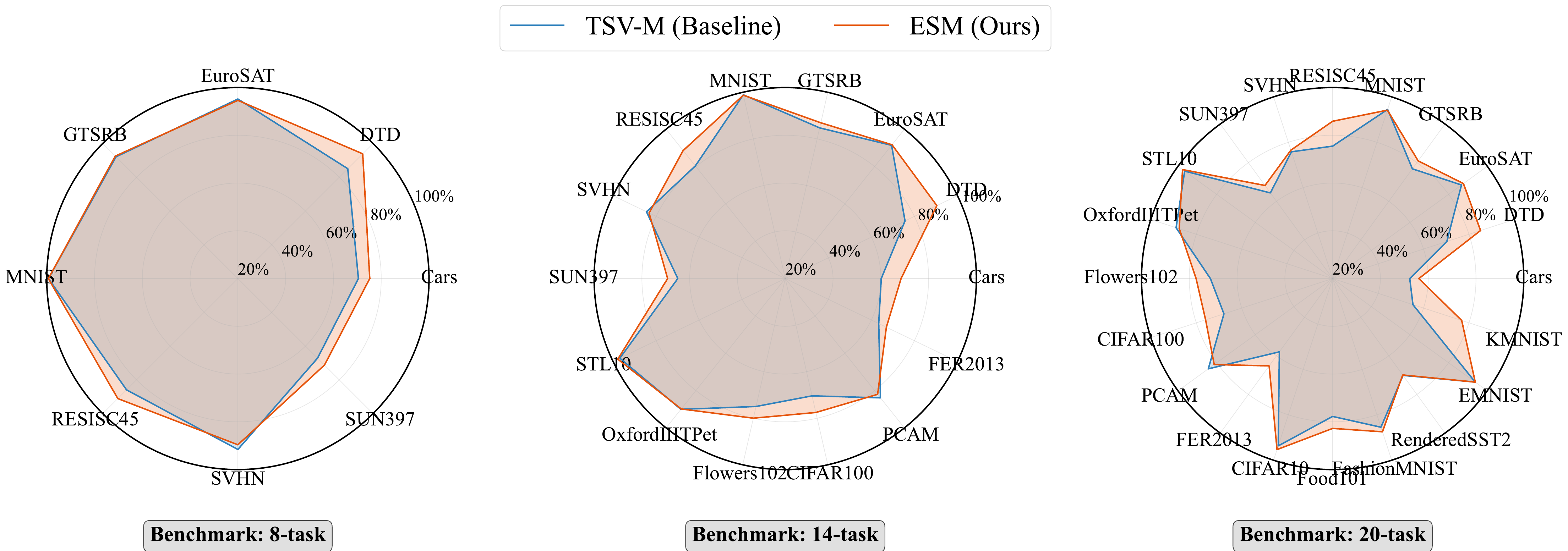}
        \caption{ViT-B/32}
    \end{subfigure}
    
    \vspace{5mm} 

    \begin{subfigure}{\linewidth}
        \centering
        \includegraphics[width=\linewidth]{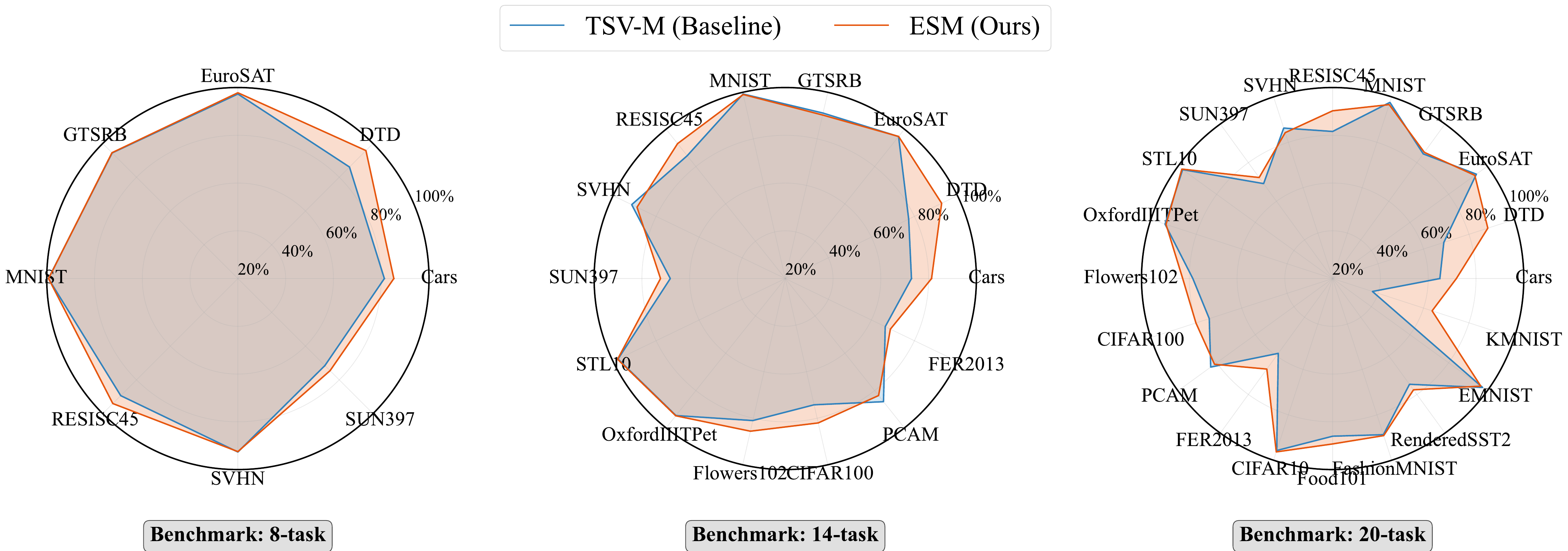}
        \caption{ViT-B/16}
    \end{subfigure}

    \vspace{5mm} 

    \begin{subfigure}{\linewidth}
        \centering
        \includegraphics[width=\linewidth]{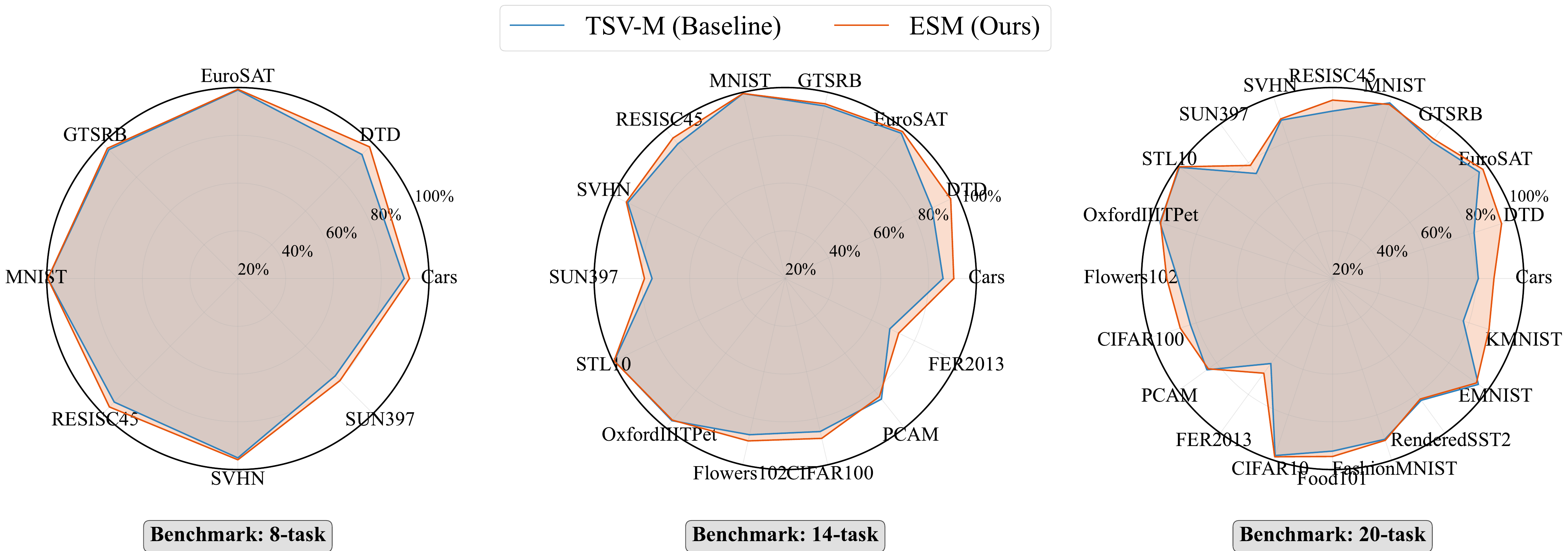}
        \caption{ViT-L/14}
    \end{subfigure}
    
    \caption{Performance comparison between the baseline method TSV-M \cite{gargiulo2025task} and our proposed ESM framework.}
    \label{fig:radar_comparison}
\end{figure*}

\begin{table*}[t]
\caption{Ablation study on the order of the three-level scaling in Polarized Scaling.}
\label{tab:ablation_scaling_order}
\setlength{\tabcolsep}{4pt}
\begin{tabular}{cccccccc}
\toprule[1.5pt] 
\multirow{2}{*}{Scaling Order} & \multicolumn{3}{c}{ViT-B/32} & \multicolumn{3}{c}{ViT-B/16}
\\
\cmidrule[0.5pt](lr){2-4} \cmidrule[0.5pt](lr){5-7} & 8 tasks & 14 tasks & 20 tasks & 8 tasks & 14 tasks & 20 tasks
\\
\midrule[0.5pt] 

\textbf{(1) Inter-Dimension, (2) Inter-Task}, (3) Inter-layer & 88.0$_{(94.1)}$ & 83.2$_{(91.4)}$ & 80.8$_{(88.3)}$ & 91.6$_{(96.7)}$ & 87.1$_{(93.7)}$ & 84.7$_{(90.9)}$ \\

\rowcolor[HTML]{C6E2FF}\textbf{(1) Inter-Task, (2) Inter-Dimension}, (3) Inter-layer & 88.4$_{(95.3)}$ & 83.7$_{(92.0)}$ & 81.3$_{(88.9)}$ & 91.8$_{(97.0)}$ & 87.4$_{(94.1)}$ & 84.9$_{(91.1)}$ \\





\bottomrule[1.5pt] 
\end{tabular}
\end{table*}

\subsection{Order of Scaling Across Tasks, Dimensions, and Layers}
As shown in Algorithm \ref{alg:esm}, the default order of Polarized Scaling in our experiments is: first across tasks, then across dimensions, and finally across layers. Since inter-layer scaling is designed to emphasize task consensus, it is applied by default to the final fused task matrices. The order of inter-task and inter-dimension scaling, however, can be interchanged. Therefore, we ablate the sequence of these two operations in Table \ref{tab:ablation_scaling_order}. The results indicate that our default order, which applies inter-task scaling first followed by inter-dimension scaling, yields better performance.

\subsection{Ablation Study on Rank \texorpdfstring{$\boldsymbol{k}$}{} Selection}
In our method and experiments, the default setting uses $k = \lfloor d_{\text{out}} / T \rfloor$ as the rank budget for low-rank decomposition of each task matrix, where $T$ denotes the number of tasks and $d_{\text{out}}$ represents the original output dimension. We conduct an ablation study on the selection of rank $k$, as shown in Figure \ref{fig:ablation_rank_k}. The results demonstrate that the merged model exhibits robustness to the choice of rank $k$, maintaining comparable performance across a wide range of values ($\lfloor 0.5\cdot d_{\text{out}} / T \rfloor \sim \lfloor 2.0\cdot d_{\text{out}} / T \rfloor$).

\begin{figure}[t]
    \centering
    \begin{subfigure}{0.49\linewidth}
        \centering
        \includegraphics[width=\linewidth]{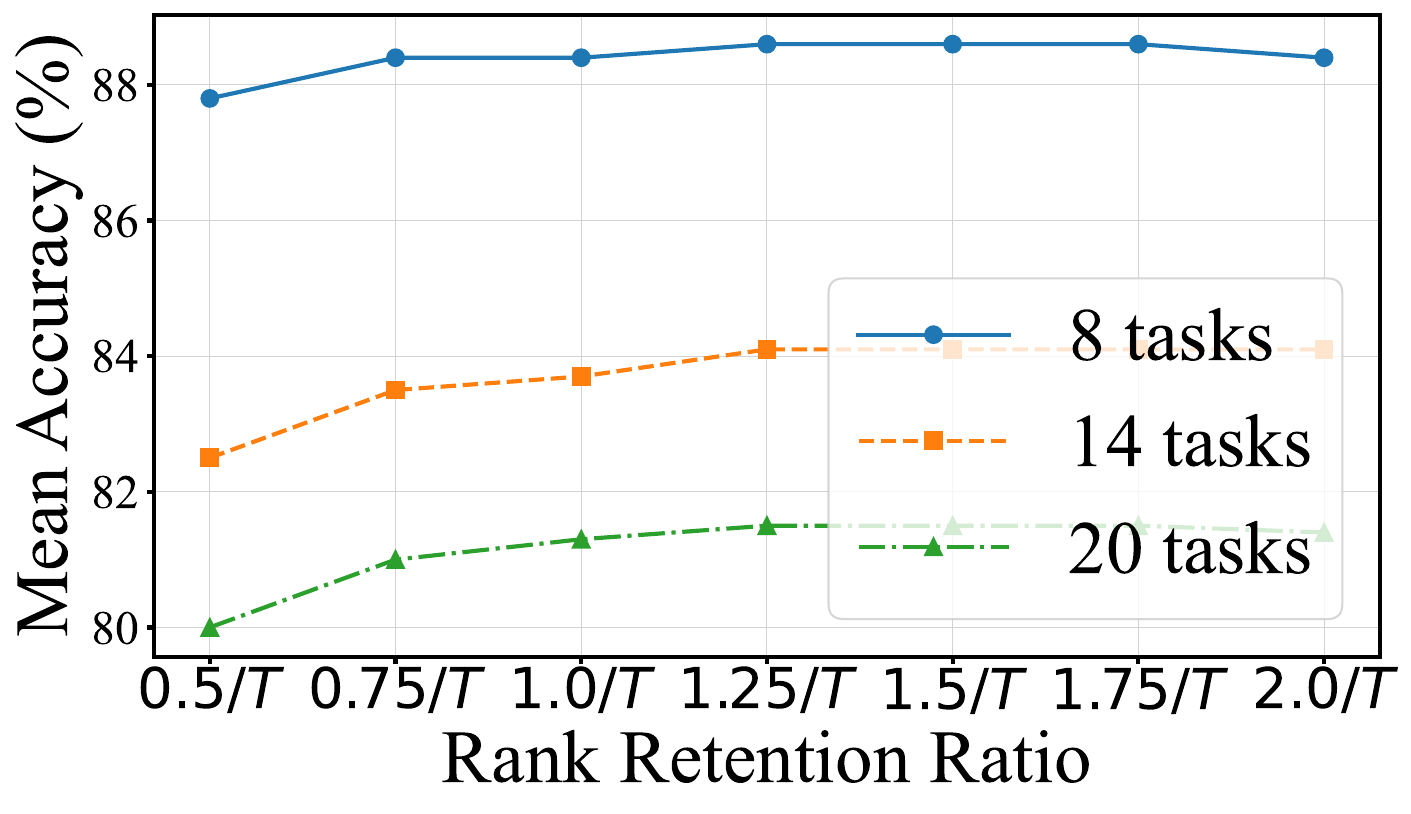}
        \caption{ViT-B/32}
    \end{subfigure}
    \hfill
    \begin{subfigure}{0.49\linewidth}
        \centering
        \includegraphics[width=\linewidth]{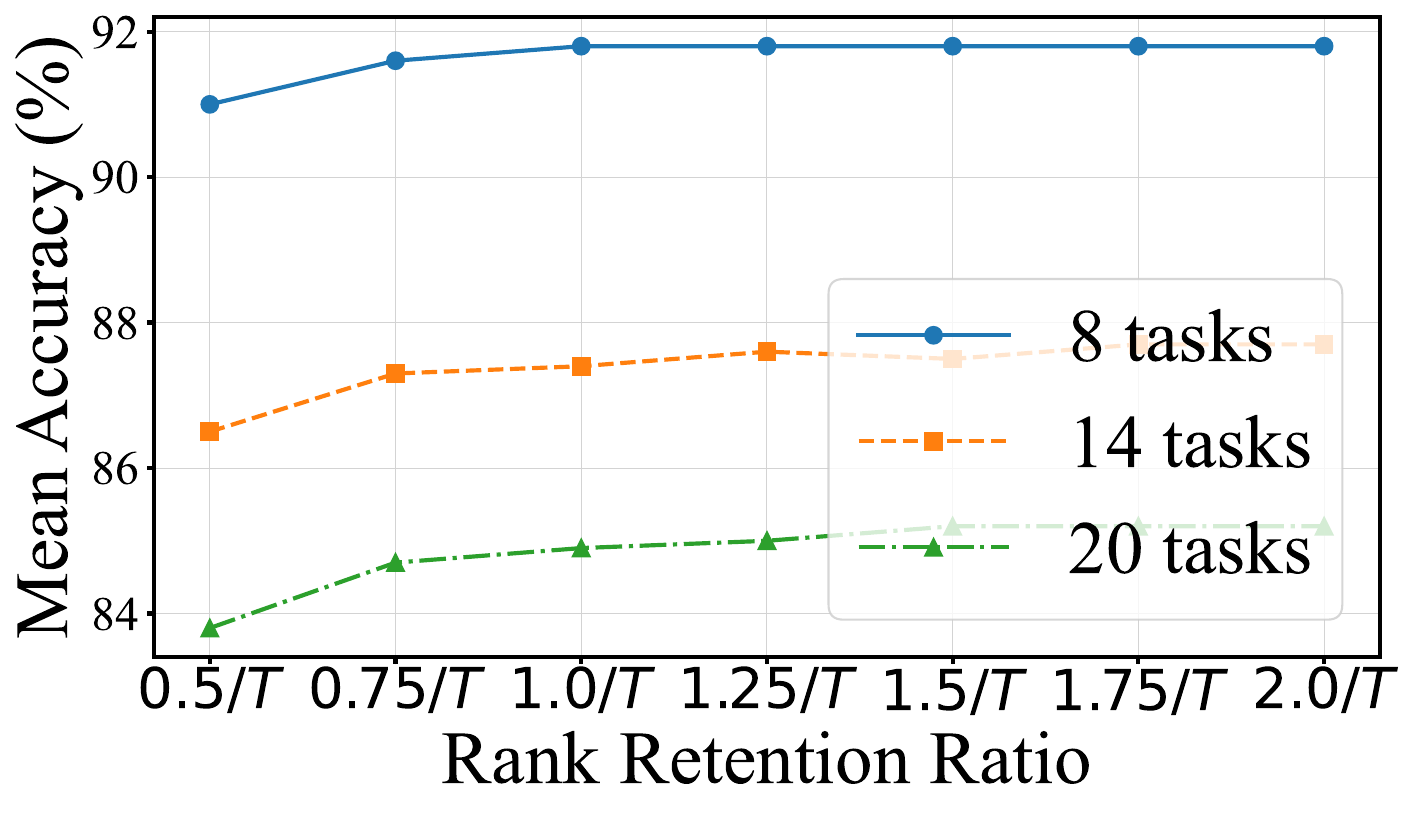}
        \caption{ViT-B/16}
    \end{subfigure}
    \caption{Ablation study on the impact of rank retention ratio on merged model performance. $T$ denotes the number of tasks.}
    \label{fig:ablation_rank_k}
\end{figure}

\subsection{Ablation Study on the Exponent of the Scaling Factor}
The default configuration of our method employs a power of $2$ in the polarized scaling coefficient, i.e., $(\frac{\text{norm}}{\mathbb{E}[\text{norm}]})^2$. The rationale for this choice is to amplify significant parameters while suppressing redundant ones. To validate the sensitivity of our approach to this hyperparameter, we conducted an ablation study. As shown in Figure \ref{fig:ablation_scaling_power}, the results indicate that model merging is robust across a range of exponents. The value of $2$ was chosen as the default because it achieves optimal performance.

\begin{figure}[t]
    \centering
    \begin{subfigure}{0.49\linewidth}
        \centering
        \includegraphics[width=\linewidth]{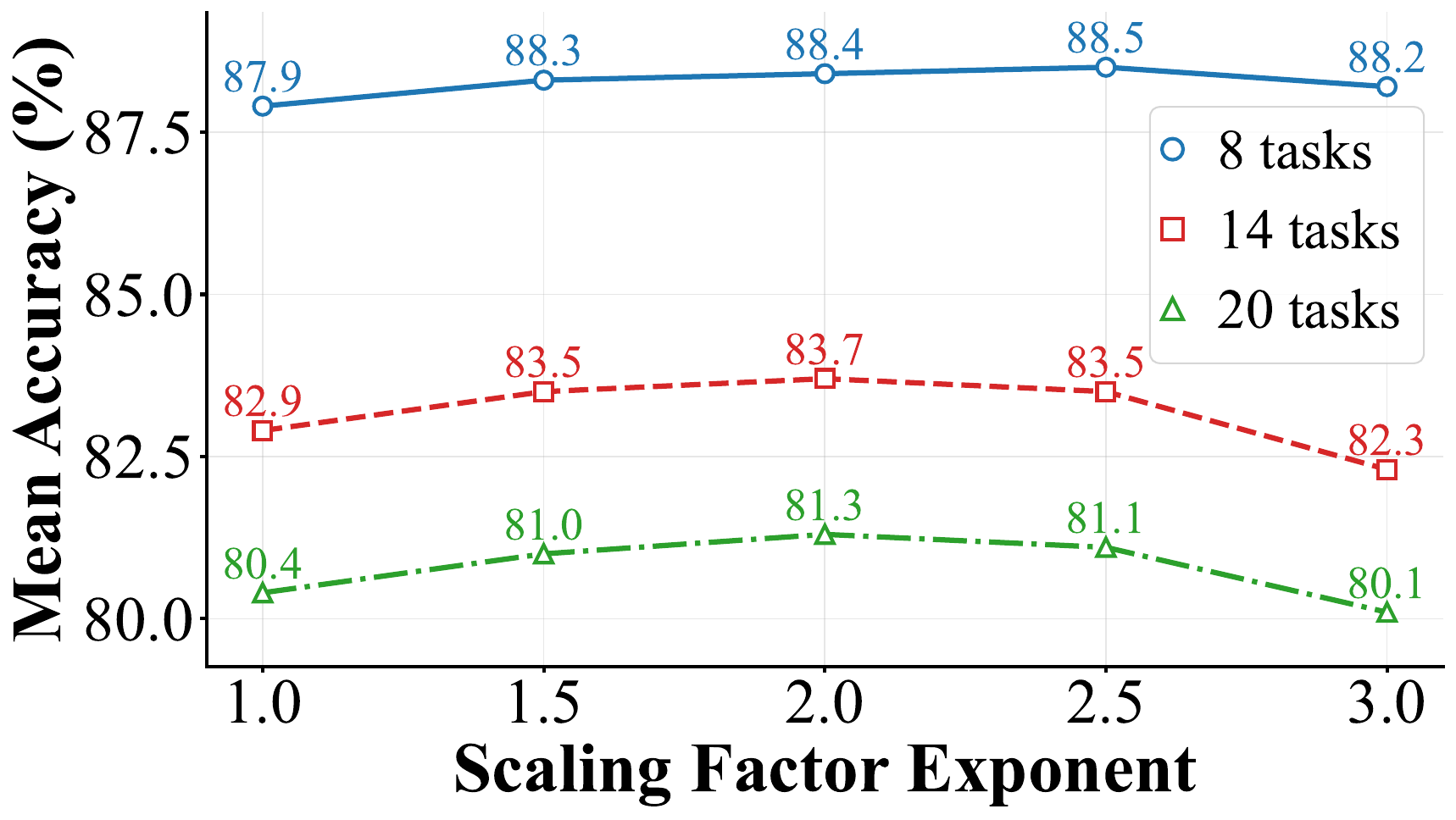}
        \caption{ViT-B/32}
    \end{subfigure}
    \hfill
    \begin{subfigure}{0.49\linewidth}
        \centering
        \includegraphics[width=\linewidth]{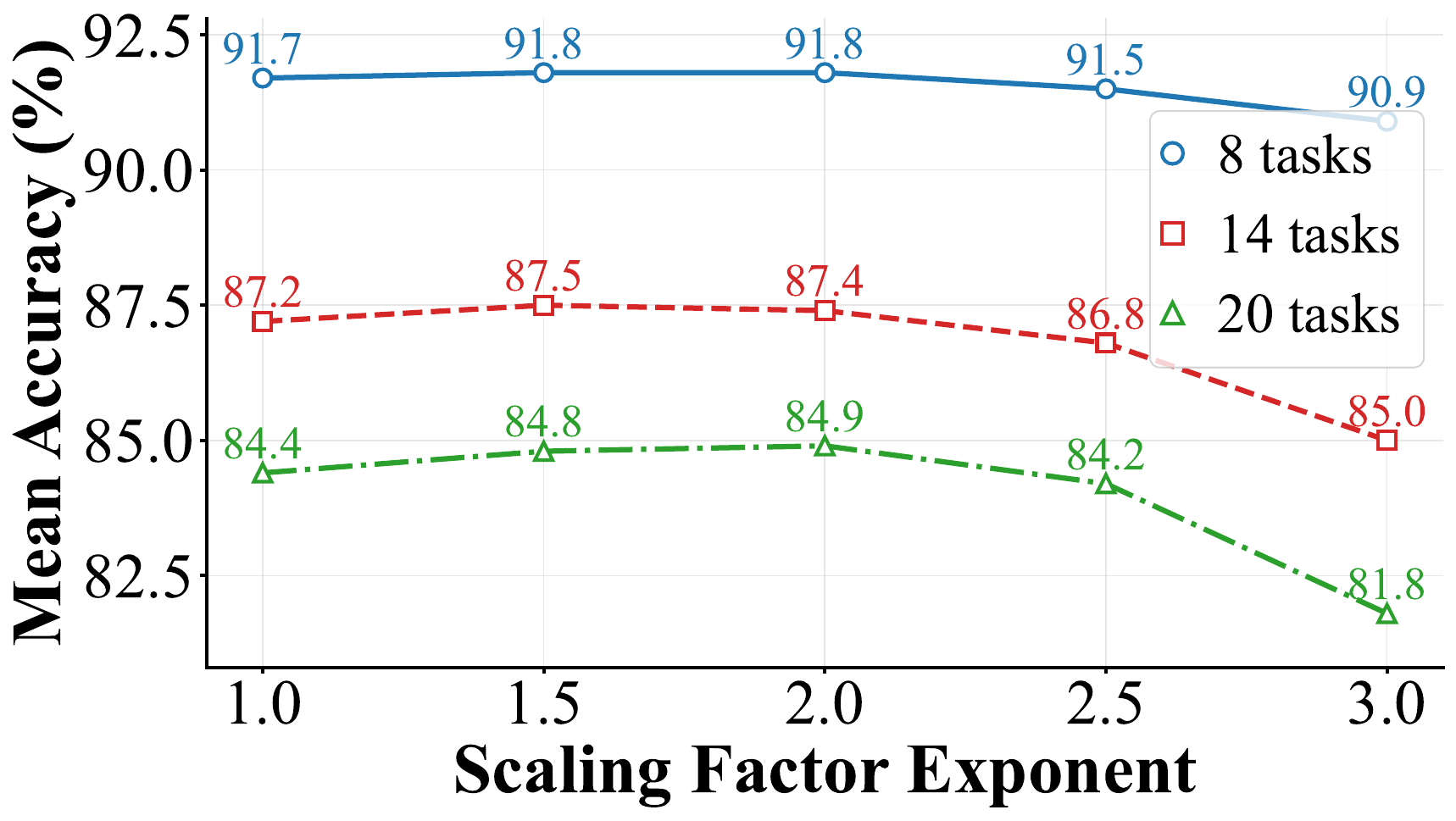}
        \caption{ViT-B/16}
    \end{subfigure}
    \caption{Impact of the scaling factor exponent on merged model performance.}
    \label{fig:ablation_scaling_power}
\end{figure}

\subsection{Calculation of Energy Retention}
\label{sec:energy_retention}
For the SVD-based method, energy retention is calculated as the ratio of the sum of squares of the retained singular values to the sum of squares of all singular values. For our ESD method, it is defined as the ratio of the sum of the retained eigenvalues to the sum of all eigenvalues. This is because the square of a singular value and an eigenvalue both correspond to the explained variance.